\documentclass{article}
\usepackage[utf8]{inputenc} 
\usepackage[T1]{fontenc}    
\usepackage{microtype}
\usepackage{graphicx}
\usepackage{subfigure}
\usepackage{gensymb}
\usepackage{booktabs} 
\usepackage{wrapfig}
\usepackage{amsmath,amssymb,amsfonts,amsthm}
\usepackage{mathtools}
\usepackage{hyperref}
\usepackage{enumitem}
\usepackage{algorithm}
\usepackage{algorithmic}
\usepackage[T1]{fontenc}
\usepackage[most]{tcolorbox}
\usepackage{lmodern} 
\usepackage{lipsum}

\usepackage{caption}
\newtcolorbox{outcome}{floatplacement=ht, float}
\newcommand*{\affmark}[1][*]{\textsuperscript{#1}}
\newcommand*\samethanks[1][\value{footnote}]{\footnotemark[#1]}
\newtheorem{theorem}{Theorem}
\newtheorem{prop}{Proposition}
\newtheorem{lemma}{Lemma}
\newtheorem{assumption}{Assumption}
\newtheorem{corollary}{Corollary}
\newtheorem{define}{Definition}
\usepackage{natbib}
\PassOptionsToPackage{numbers, compress}{natbib}
\usepackage{arxiv}
\usepackage{paralist}

\newcommand{\dtv}{d_{\text{TV}}}
\newcommand{\distp}{\mathcal{P}}
\newcommand{\distq}{\mathcal{Q}}
\newcommand{\defeq}{\vcentcolon=}

\title{Online Continual Adaptation with Active Self-Training}
\author{
\parbox{\linewidth}{\centering
Shiji Zhou\textsuperscript{\rm 1},
Han Zhao\textsuperscript{\rm 2},
Shanghang Zhang\thanks{Corresponding Authors} \textsuperscript{\rm 3},
Lianzhe Wang\textsuperscript{\rm 1},
Heng Chang\textsuperscript{\rm 1},
Zhi Wang\textsuperscript{\rm 4},
Wenwu Zhu\samethanks{} \textsuperscript{\rm 5}} \\
}

\begin{document}

\maketitle
\vspace{-4em}
 \begin{center}
 \affmark[1]Tsinghua-Berkeley Shenzhen Institute, \affmark[4]Tsinghua Shenzhen International Graduate School, \\ and \affmark[5]Department of Computer Science and Technology, Tsinghua University \\ 
 \affmark[2]Department of Computer Science, University of Illinois at Urbana-Champaign\\
 \affmark[3]Berkeley AI Research (BAIR), University of California, Berkeley\\
 \texttt{\{zsj17,wanglz20,changh17\}@mails.tsinghua.edu.cn, hanzhao@illinois.edu, shz@eecs.berkeley.edu, wangzhi@sz.tsinghua.edu.cn, wwzhu@tsinghua.edu.cn}
 \end{center}
 \vspace{1em}

\begin{abstract}
  Models trained with offline data often suffer from continual distribution shifts and expensive labeling in changing environments. This calls for a new online learning paradigm where the learner can continually adapt to changing environments with limited labels. In this paper, we propose a new online setting -- Online Active Continual Adaptation, where the learner aims to continually adapt to changing distributions using both unlabeled samples and active queries of limited labels. To this end, we propose Online Self-Adaptive Mirror Descent (OSAMD), which adopts an online teacher-student structure to enable online self-training from unlabeled data, and a margin-based criterion that decides whether to query the labels to track changing distributions. Theoretically, we show that, in the separable case, OSAMD has an $O({T}^{2/3})$ dynamic regret bound under mild assumptions, which is aligned with the $\Omega(T^{2/3})$ lower bound of online learning algorithms with full labels. In the general case, we show a regret bound of $O({T}^{2/3} + \alpha^* T)$, where $\alpha^*$ denotes the separability of domains and is usually small. Our theoretical results show that OSAMD can fast adapt to changing environments with active queries. Empirically, we demonstrate that OSAMD achieves favorable regrets under changing environments with limited labels on both simulated and real-world data, which corroborates our theoretical findings.
\end{abstract}

\section{Introduction}
Machine learning models, trained with data collected from closed environments, often suffer from continual distribution shift and expensive labeling in open environments. For example, a self-driving recognition system trained with data collected in the daytime may continually degrade when going towards nightfall~\citep{bobu2018adapting,wu2019ace}. The problem can be avoided by collecting and annotating sufficient training data to cover all the possible distributions at the test time. However, such data annotation is prohibitively expensive in many applications~\citep{zhang2020collaborative}. In particular, for many scenarios, the distribution shifts constantly appear over time~\citep{kumar2020understanding}, making it impossible to collect and annotate sufficient training data for a certain domain. This calls for a new online system that can continually adapt to the changing domain using limited labels.

The continual domain shift severely challenges the conventional domain adaptation methods~\citep{tzeng2014deep,ganin2015unsupervised,hoffman2018cycada}, for most of them are designed to adapt to a fixed target domain~\citep{Su_2020_WACV,prabhu2021active} (Figure~\ref{fig:Intro} bottom). Some previous works consider gradual domain shift~\citep{bobu2018adapting,wu2019ace,kumar2020understanding}, where the data distribution gradually evolves from batch to batch, but it is not realistic to model the continual shift that happens in continuous time. The adaptive online learning~\citep{besbes2015non} provides a classical theoretical framework to deal with changing environments. However, it requires the target data to be fully labeled (Figure~\ref{fig:Intro} middle), which may be infeasible. Furthermore, it remains an open problem for online learning with limited labels (online active learning) under continual domain shift~\citep{lu2016online,shuji2017budget}. Recent work~\citep{chen2021active} studies online active domain adaptation for regression problem under covariant (i.e. $P(X)$) shift, but can not deal with classification problem under joint distribution (i.e. $P(X,Y)$) shift, which is more general and realistic~\citep{long2017deep,long2018conditional}.

\begin{figure}
    \centering
    \includegraphics[width=0.6\linewidth]{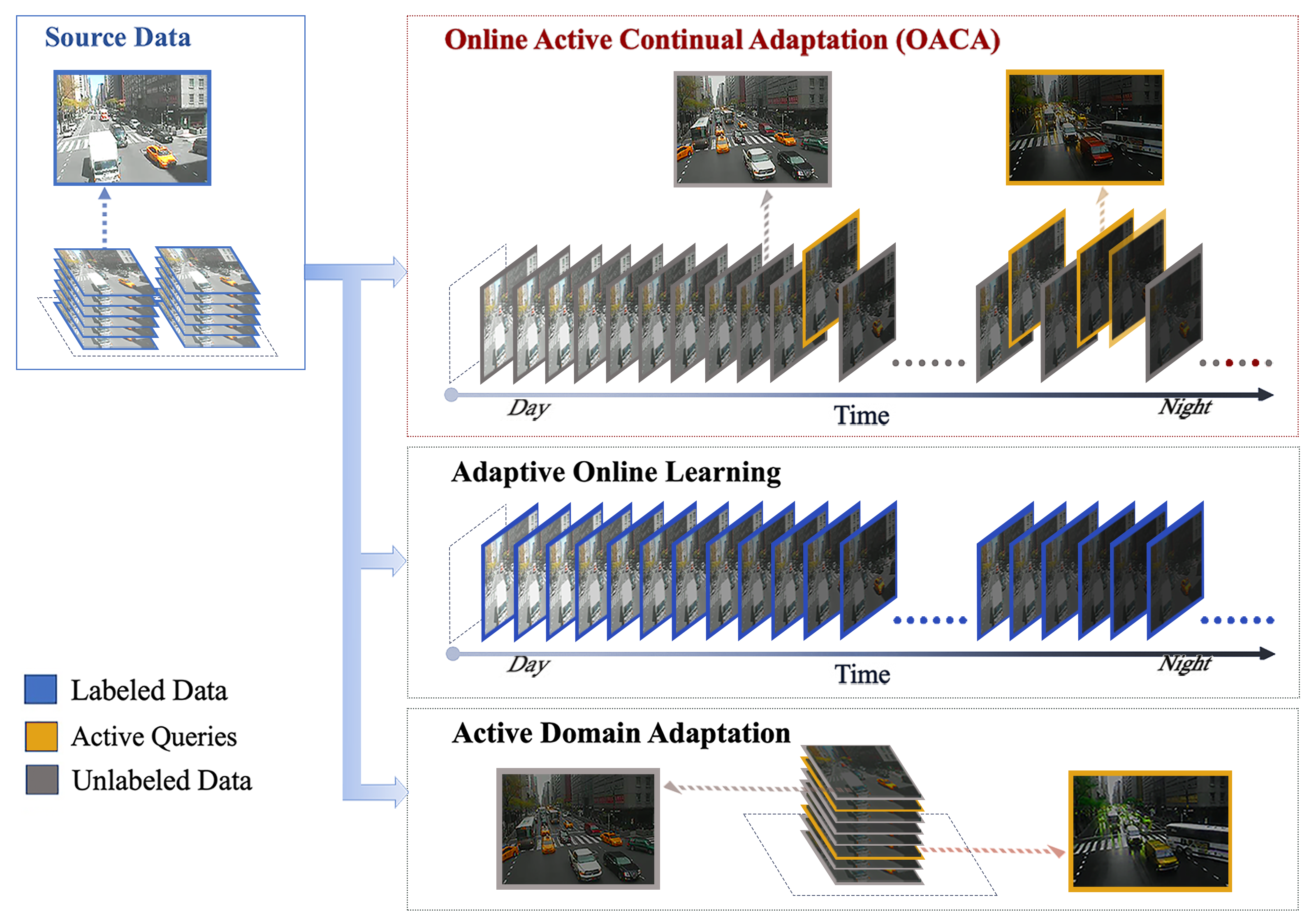}
    \caption{Illustration of OACA and previous works. The environments change continually from daytime to nightfall, introducing a shift of illumination and weather conditions. We model such a problem as an online active continual adaptation, where the online learner adapts to the environment with active queries of limited labels. In contrast, adaptive online learning requires full labels, and active domain adaptation can only adapt to a fixed domain.
    }
    \label{fig:Intro}
\end{figure}

To the best of our knowledge, no previous work has considered the online continual adaptation for classification with limited labels\footnote{Please refer to the detailed discussion in our literature view in Appendix. \ref{sec:related}}. To fill this gap, we formulate the Online Active Continual Adaptation (OACA) problem, where learners start with initial models and aim to minimize the dynamic regret caused by the distributional shifts using unlabeled data and active queries.

To resolve this problem, we propose the Online Self-Adaptive Mirror Descent (OSAMD) algorithm, which adopts an \emph{online teacher-student structure} to enable the \emph{self-adaptation} from unlabeled data: an ``aggressive'' model that updates actively using limited label queries with aggressive stepsizes, to track the max-margin classifier and provide accurate pseudolabels; a ``conservative'' model that adapts continually using the pseudolabels (taught by the aggressive model) with conservative stepsizes, to track the domain and minimize the dynamic regret. The active queries are given by a margin-based strategy that measures the confidence of the pseudolabels to query uncertain samples.

Theoretically, we show that OSAMD achieves an $O(T^{2/3})$ dynamic regret under mild assumptions in the separable case, which is aligned with the $\Omega(T^{2/3})$ lower bound of full-label online algorithms. We then extend this result to the general non-separable case, and derive a dynamic regret of $O(T^{2/3} + \alpha^* T),$ where $\alpha^*$ represents the separability of the data distribution. Since $\alpha^*$ is often a small constant by the representation ability of machine learning models~\citep{kumar2020understanding}, the bias $O(\alpha^* T)$ is small such that the regret still approximates the lower bound. The above results lead to algorithmic insights that OSAMD is competitive with the optimal model in hindsight with fast convergence. 

Empirically, we establish a simulation with changing environment to corroborate our theoretical findings. The results show that OSAMD performs accurately with limited labeled samples, and the regret aligns with Online Mirror Descent (OMD) with full labels. Furthermore, we extend OSAMD to deep learning on the real-world datasets Portraits~\citep{ginosar2015century} and Cover-Type~\citep{blackard1999comparative} to verify its practical effectiveness. OSAMD attains $93.7\%$(Portraits)/$76.8\%$(Cover-Type) accuracy using only $3.8\%$/$2.3\%$ labels, comparing with $94.0\%$/$76.8\%$ accuracy of OMD with full labels. While OMD with uniform query and online active learning baseline only obtain $91.8\%$/$75.3\%$ and $92.0\%$/$76.5\%$ with $3.8\%$/$2.3\%$ queries, respectively. Finally, our ablation study shows that both the self-adaptation and active strategy contribute to the remarkable performance. 

\paragraph{Our contributions}
\begin{enumerate}
    \item   We are the first to formulate the problem of Online Active Continual Adaptation, which models the online continual domain adaptation with limited labels under realistic assumptions; 
    \item   We propose an effective online self-adaptive algorithm named OSAMD with the novel design of the online teacher-student structure supported by strong theoretical guarantees;
    \item   We provide strict and novel dynamic regret analysis, highlighting the power of the proposed online teacher-student structure to get tight regret bound as in the regime of full-label online learning; 
    \item   We demonstrate the effectiveness of our algorithm on both simulated and real-world datasets with changing environments, corroborating our theory.
\end{enumerate}

\section{Preliminaries}\label{sec:preliminary}
We first give a formal definition of the online convex optimization (OCO) with dynamic regret, and then briefly review the classic online mirror descent (OMD) algorithm. Next, we introduce the problem setting of domain adaptation (DA) and a statistical distance that will be used in our analysis.

\subsection{Online Convex Optimization with Dynamic Regret}
The basic protocol of Online Convex Optimization (OCO)~\citep{hazan2016introduction} is: at each time step $t=1,\dots T,$ the online learner takes a decision $w_t$ in a convex set $\mathcal{K}.$ After that, the environment reveals a convex loss function $l_t: \mathcal{K} \rightarrow \mathbb{R},$ and the online learner suffers a loss $l_t(w_t).$ The dynamic regret is a theoretical metric for an online algorithm in changing environments, defined as
\begin{align*}
    &\text{D-Regret} \defeq \sum_{t=1}^T l_t(w_t) - \sum_{t=1}^T l_t(w_t^*),
\end{align*}
where $w_t^* = \arg\min_{w\in\mathcal{K}} l_t (w).$ It is well known that a sublinear dynamic regret bound is not possible unless specific constraints are made about the environments~\citep{zinkevich2003online}. The first type of such constraint is the \emph{path-length}~\citep{zinkevich2003online,hall2013dynamical}: $C_T \defeq \sum_{t=1}^{T-1} \|w_t^* - w_{t+1}^*\|,$
which measures how the optimal models change with the environment. Others include the \emph{temporal variability}~\citep{besbes2015non,campolongo2020temporal}: $V_T \defeq \
sum_{t=1}^{T-1} \sup_{w\in\mathcal{K}} \|l_t(w) - l_{t+1}(w)\|,$ which measures the variation of the loss functions. 

Online Mirror Descent (OMD)~\citep{hazan2016introduction} is a general and classic algorithm, where the decision is updated by 
\begin{equation*}
    {w}_{t+1} = \arg\min_{w\in \mathcal{K}} \eta \langle \nabla l_t({w}_{t}), w\rangle + D_{\mathcal{R}} (w,{w}_{t})
\end{equation*}
for $t=1,\ldots,T,$ where $\eta$ denotes the stepsize, and $D_{\mathcal{R}}(a,b) \defeq \mathcal{R}(a) - \mathcal{R}(b)  - \langle \nabla \mathcal{R}(b) , a - b\rangle$ denotes the Bregman divergence with regularizer ${\mathcal{R}}.$ 

\subsection{Domain Adaptation}
Domain adaptation (DA)~\citep{zhao2020review} is a typical machine learning method to learn a model from a source domain $\distp$ that can perform well on a target domain $\distq.$ Researches~\citep{wei2020theoretical} often assume that the source domain and target domain are different but measured by some discrepancies. In this paper, we utilize the celebrated total variation (TV) distance $\dtv$ to describe the similarity between two distributions over the same sample space:
\begin{define}[Total Variation]
We use $\dtv(\distp, \distq)$ to denote the total variation distance between distributions $\distp$ and $\distq$:
    \begin{equation*}
    \dtv(\distp, \distq)\defeq \sup_{E}|\distp(E) - \distq(E)|,
    \end{equation*}
where the supremum is over all the measurable events. 
\end{define}
Several variants of the TV distance have been proposed and used in the domain adaptation~\citep{ben2010theory,zhao2018multiple,zhao2019learning}. It is also worth pointing out that other metrics can and have been used in the literature, e.g., the Wasserstein infinity~\citep{kumar2020understanding} distance and the maximum mean discrepancy~\citep{long2015learning}, as described in our appendix. 

\section{The Online Active Continual Adaptation Problem}
\label{Sec:ProblemSetting}
In this section, we first formulate the online active continual adaptation (OACA) problem. We then formally introduce the assumptions used in our analysis and provide justifications for their necessities.

\subsection{Problem Formulation}
For the purpose of presentation, we consider an online binary classification task\footnote{This can be readily extended to the multi-class case, as shown in Appendix. \ref{supp:multiclass}} of sequentially predicting labels $y_t\in \{1,-1\}$ from input features $x_t\in \mathcal{X}$ for round $t=1,\dots,T$. In each round $t,$ assume that our prediction model is parameterized by a vector $w_t\in \mathcal{K}$, and it outputs a soft label prediction\footnote{The model output before the sign function, its absolute value is related to the distance to the decision boundary.} over instance $x_t$ denoted as $H_t(w_t) = H(w_t;x_t)$. The prediction result suffers a  instantaneous loss as $f_t(w_t) = f(w_t;x_t,y_t), x_t\in \mathcal{X}, y_t \in \{-1,1\}.$
We assume that $H, f$ (unrelated with $x,y$) are known by the learner. We present the interaction between the learner and the environment as Figure~\ref{fig:OACA}. Specially, we assume each data sample $(x_t, y_t)$ comes from a different distribution by the continual environmental change, which leads to different distributions at different times. 
\begin{figure}[!t]
	\centering
	\fbox{
		\begin{minipage}{0.95\linewidth}
			\begin{enumerate}[\hspace{0em}1.]
        \item The data distribution begins with $P_1 (X,Y).$
        \item The learner has enough data samples from $P_1 (X,Y),$ and chooses an online algorithm $\mathcal{A}.$
        \item The adversary chooses a sequence of data distribution $\{P_2, \dots, P_T\}.$
        \item For each $t = 1,\dots, T:$
        \begin{enumerate}
        \item The data $(x_t,y_t)$ is sampled from joint distribution $P_t (X,Y).$
        \item Instance $x_t$ is revealed to the learner.
        \item The learner then chooses action $w_t,$ incurring a loss on the domain $l_t(w_t)=\mathbb{E}_{(x,y)\sim P_t}f(w_t;x,y)$ in hindsight.
        \item The active agent decides whether to query the label. If query, true label $y_t$ are revealed.
    \end{enumerate}
\end{enumerate}
		\end{minipage}
	}
	\caption{\emph{Online Active Continual Adaptation setting.}}
	\label{fig:OACA}
\end{figure}
To measure how the learner adapts to the environment, we use the theoretical metric of expected dynamic regret to measure the adaptation performance of the online learner:
\begin{equation*}
    \text{D-Regret}^{\mathcal{A}}(\{P_t\},T) \defeq  \mathbb{E}^{\mathcal{A}}[\sum_{t=1}^T l_t(w_t)] - \sum_{t=1}^T l_t(w_t^*),
\end{equation*}
where $l_t(w_t) = \mathbb{E}_{(x,y)\sim P_t}f(w_t;x,y)$ is the expected loss, and the optimal action in hindsight is defined as $w_t^* = \arg\min_{w\in \mathcal{K}} l_t(w)$. The rest expectation $\mathbb{E}^{\mathcal{A}}$ is on the online decision $w_t$ provided by a potentially randomized algorithm $\mathcal{A}.$ 

In particular, we here use the expected loss that reflects the performance on the distribution $P_t$ not the instantaneous sample $(x_t, y_t)$. Intuitively, we study the continual domain adaptation problem, where the expected loss reflects the performance on the domain (environment), while the instantaneous loss only reflects the performance on individual samples. On the technical side, although the distributional change is continual, the change between consecutive samples could be large due to the randomness in sampling, leading to an unbounded dynamic regret~\citep{besbes2015non}.

\subsection{Assumptions}
First, we assume the domain shift is continual and bounded.
\begin{assumption}[Continual Domain Shift]\label{amp:gradualshift} 
    There exists a constant $V_T,$ s.t.$\sum_{t=1}^{T-1}d_{TV}\left(P_t, P_{t+1}\right) \leq {V_T}.$ In other words, the total domain shift is bounded.
\end{assumption}
This assumption is closed related to the temporal variability constrain~\citep{besbes2015non}, where we specify the expectation change as domain shift.

Next, we assume a niceness condition of the environments by its separability.
\begin{assumption}[Separation]\label{amp:separable}
    For each time step $t\in[T]$, the data distribution $P_t$ can be classified almost correctly with a margin $R$, i.e., there exists $v_t \in \mathcal{K}$ and a constant $\alpha^*$ such that $\mathbb{E}_{(x_t,y_t)\sim P_t} [\max\{0,R - y_t H_t(v_t)\}] \leq \alpha^*.$
    Furthermore, there exists a constant $C_T$ such that
    $\sum_{t=1}^{T-1}\|v_t-v_{t+1}\| \leq C_T, $
    i.e., the classifiers with margin $R$ change continually.
\end{assumption}
Note that the coefficient $\alpha^*$ represents the optimal hinge loss for the classifier space. It is commonly assumed to be small by previous works \citep{kumar2020understanding,wei2020theoretical}. The constraint of $\sum_{t=1}^{T-1}\|v_t-v_{t+1}\|$ is similar to the path-length regularity~\citep{zinkevich2003online} in online learning. Intuitively, $v_t$ can be viewed as a max-margin classifier, and the continual rotation is bounded by $C_T$. 

Finally, we present the following standard assumptions in online learning.
\begin{assumption}[Convexity]\label{amp:convexity}
    We assume that $f(\cdot),-y H(\cdot)$ are all convex functions.
\end{assumption}

\begin{assumption}[Smoothness]\label{amp:boundedgradient}
    We assume that $f$ and $H$ are differentialable and $G$-Lipschitz, i.e. $\|\nabla f(w;x,y)\|_*, \|\nabla H(\theta; x)\|_* \leq G, \forall x,y,w,$ where $\|\cdot\|_*$ is the dual norm of $\|\cdot\|.$ Furthermore, f is $L$-smooth, i.e. $\|\nabla f(w) - \nabla f(w')\|\leq L\|w-w'\|.$
\end{assumption}

\begin{assumption}[Bounded Decision Space and Function]\label{amp:boundeddecisionspace}
    The diameter of decision space $\mathcal{K}$ (convex set in $\mathbb{R}^n$) is bounded, i.e. there exists $D > 0$ such that $\max_{w,w'\in \mathcal{K}}\|w - w'\|\leq D.$ The function value is bounded, i.e. there exists $F > 0$ such that $f(w;x,y)\leq F, \forall w,x,y.$
\end{assumption}

\section{The Online Self-Adaptive Mirror Descent Algorithm}
\label{Sec:Algorithm}

Here we describe the proposed Online Self-Adaptive Mirror Descent (OSAMD) in Algorithm~\ref{algorithm:OAST}. To make our description easier to follow, we first introduce the following procedures.
\begin{enumerate}[\hspace{0em}1.]
    \item \textbf{Pseudolabel:} Pseudolabel the example $x_t$ by an aggressive model (parameterized by $\theta$).
    \item \textbf{Self-adaptation:} Before making the decision, the learner trusts the pseudolabel and self adapts the conservative model (parameterized by $w,\hat{w}$) by implicit mirror descent. 
    \item \textbf{Active query:} The active agent decides whether to query the label based on the margin measured by $|H_t (\theta)|.$ If query, update the aggressive model by mirror descent with the true label by adaptive stepsizes. The conservative model is updated with pseudolabel or queried label by mirror descent.
\end{enumerate}
\begin{algorithm*}[tb]
   \caption{Online Self-Adaptive Mirror Descent}
   \label{algorithm:OAST}
\begin{algorithmic}
   \STATE {\bfseries Input:} Active probability controller $\sigma,$ aggressive step size $\tau_t,$ conservative step size $\eta,$ initial data.
   \STATE {\bfseries Initial:} Learn from initial data, get aggressive model $\theta_1$ and conservative model $\hat{w}_1 $.
   \FOR {$t = 1,\dots, T$ }
   \STATE observe data sample $x_t$ 
   \STATE \textbf{pseudolabel:} 
   \STATE \quad give the pseudolabel provided by the aggressive model $\hat{y}_t = sign(H_t(\theta_t))$
   \STATE \textbf{self-adaptation:} 
   \STATE \quad adapt the conservative model $w_t = \arg\min_{w\in \mathcal{K}} \eta f(w;x_t,\hat{y}_t) + D_{\mathcal{R}} (w,\hat{w}_{t})$, and then make the decision
   \STATE \textbf{active query:} 
   \STATE \quad draw a Bernoulli random variable with probability $Z_t\sim Bernoulli(\sigma / (\sigma + |H_t(\theta_t) |))$
   \IF {$Z_t = 1$}
   \STATE query label $y_t,$ and let $\tilde{y}_t = {y}_t$
   \STATE update the aggressive model by  $\theta_{t+1} = \arg\min_{\theta\in \mathcal{K}} - \tau_t \langle y_t \nabla H_t(\theta_t), \theta\rangle + D_{\mathcal{R}} (\theta,\theta_t)$
   \ELSE 
   \STATE let $\theta_{t+1} = \theta_{t}$ and $\tilde{y}_t = \hat{y}_t$
   \ENDIF
   \STATE update the conservative model by $\hat{w}_{t+1} = \arg\min_{w\in \mathcal{K}} \eta \langle \nabla f({w}_{t};x_t, \tilde{y}_t), w\rangle + D_{\mathcal{R}} (w,\hat{w}_{t})$ 
   \ENDFOR
\end{algorithmic}
\end{algorithm*}

\paragraph{Intuitive description}
At a high level, we design an ``aggressive'' model to track the max-margin classifier in order to produce correct pseudolabels. On the other hand, the conservative model is updated with pseudolabels with the goal of minimizing the dynamic regret. Intuitively, the aggressive model is updated with an aggressive stepsize, thus can track the max-margin classifier by limited updates with the true labels and provide trustful pseudolabels, although its regret might be large. The trustful pseudolabels enable the learner to ``look ahead'' with the incoming label and self-adapt the conservative model before making the final decision, leading to a lower regret. Our active query agent measures the uncertainty of pseudolabels by the margin between the data samples and the decision boundary, i.e., $|H_t (\theta)|,$ and tends to query the uncertain samples. Then update the aggressive model with the active queries in time. Finally, the conservative model is updated with highly confident pseudolabels or queried real labels by a fixed conservative stepsize, and thus keeps tracking the continual domain shift.

\paragraph{Online teacher-student structure}
Our novel design of running two models $\theta$ and $w,$ where the aggressive model $\theta$ teaches the adaptation of the conservative model $w$, is motivated by the special property of continual domain shifts. Specifically, the max-margin classifier shifts continually and will not ``cross over'' the margin in a short time, thus the aggressive model does not need to update frequently, since the old model still can give correct pseudolabels. When the max-margin classifier ``crosses over'' the margin, the active agent will detect the uncertainty, then the aggressive model $\theta$ needs to track the max-margin classifier shift with an ``aggressive'' stepsize. On the other hand, as the continual domain shift leads to the continual change of the expected loss (on the distribution) $l_t,$ the optimal minimizer $w_t^*$ evolves continually, leading to the need for a conservative model $w$ that frequently adapts to the shift using a ``conservative'' stepsize.

\paragraph{Remark} Our pseudolabel and self-adaptation in the online setting, where we adopt a teacher-student structure, are different from the existing works in the offline setting~\citep{zhao2020review}. Also, there is no previous online learning algorithm that uses the same ideas to the best of our knowledge. In fact, this is an open question for online active learning under non-stationary scenarios~\citep{lu2016online,shuji2017budget}.

\section{Theoretical Analysis}
\label{Sec:Analysis}
In this section, we analyze the dynamic regret bound of the proposed algorithm. We first begin with the separable case, and theoretically show the necessity of active query and the efficacy of the online teacher-student structure. Finally, we generalize the results to the non-separable case. 
\subsection{Separable Case}
We begin with analyzing the regret bound in the separable case, i.e., $\alpha^*=0.$ Since the pseudolabel errors are highly related to the regret bound, we first present the following lemma:

\begin{lemma}[Pseudolabel Errors]\label{thm:labelError}
  Let the regularizer $\mathcal{R}: \mathcal{K} \mapsto \mathbb{R}$ be a 1-strongly convex function on $\mathcal{K}$ with respect to a norm $\|\cdot\|.$ Assume that $D_{\mathcal{R}}(\cdot, \cdot)$ satisfies $D_{\mathcal{R}} (x,z)-D_{\mathcal{R}}(y,z) \leq \gamma \|x-y\|, \forall x,y,z \in \mathcal{K}.$ Set 
  $
      \tau_t = \frac{\max\{0,\sigma-y_t H_t(\theta_t)\}}{\| \nabla H_t(\theta_t)\|^2_*},\text{ and }\sigma \leq R.
  $
  If $\alpha^*=0,$ the expected number of pseudolabel errors made by Algorithm~\ref{algorithm:OAST} is bounded by
  \begin{align*}
        \mathbb{E}[\sum_{t=1}^T M_t] \leq  \frac{2G^2}{\sigma^2} (\gamma C_T + \epsilon_v) ,
    \end{align*}
    where $M_t = \mathbf{1}_{\hat{y}_t \neq y_t}$ is the instantaneous mistake indicator, and $\epsilon_v = D_{\mathcal{R}}(\theta_1, v_1)$.
\end{lemma}
We provide detailed proof in the appendix, where we refer to and generalize the technique of online active learning~\citep{cesa2006worst,lu2016online} in stationary settings with $l_2$ regularizer. Our results hold for non-stationary scenarios with any regularizers $\mathcal{R},$ which solves the open question proposed in~\citep{lu2016online,shuji2017budget}.

As illustrated in Lemma~\ref{thm:labelError}, we have three observations: 1) \emph{Higher query rate leads to lower errors bound.} From the upper bound, the expected mistakes bound is inversely proportional to query probability controller $\sigma$; 2) The \emph{pseudolabel mistakes are bounded by the classifier shift} $C_T.$ This is aligned with the intuition that if the max-margin classifier shifts severely, then the pseudolabel agent is more likely to make mistakes; 3) \emph{Better initialization implies fewer errors.} Better initialization leads to small $\epsilon_v,$ and thus implies a tighter upper bound, which shows the importance of the pre-trained model. As we assume both $C_T$ and $\epsilon_v$ are constants, the expected errors are small and controllable. 

It should be noted that $\mathbb{E}[f(w_t;x_t,y_t)|w_t]\neq l_t (w_t),$ since $w_t$ depends on $x_t,$ such that fixing $w_t$ changes the distribution of $(x_t,y_t).$ This leads to biased gradients that complicate the regret analysis. Thus, before presenting the regret, we first measure the impact of the bias brought by this dependency.

\begin{lemma}\label{lem:noiseGradient}
  For algorithm~\ref{algorithm:OAST}. We have the following inequality for $u_t\in\mathcal{K},t= 1, \dots, T $
  \begin{align*}
    \mathbb{E} [ l_t (w_t) - l_t(u_t)] \leq \mathbb{E} [\langle \nabla f(w_t;x_t,y_t), w_t - u_t \rangle] + \mathbb{E}[2(LD+G)\|w_t - \hat{w}_t\|].
  \end{align*}
\end{lemma}

Lemma~\ref{lem:noiseGradient} shows that the impact of the gradient bias is bounded by $\|w_t - \hat{w}_t\|$, which is controlled by the choice of the stepsize $\eta$. Now, we are ready to analyze the dynamic regret bound.

\begin{theorem}[Dynamic Regret]\label{thm:ExpectedRegretNoiseGradient}
  Under the same conditions and parameters in Lemma~\ref{thm:labelError}, Algorithm~\ref{algorithm:OAST} achieves the following dynamic regret bound
    \begin{align*}
      {\text{\rm D-Regret}}^{\text{OSAMD}}(\{P_t\},T) \leq  \frac{4(\eta G^4 + G^3 D)}{\sigma^2} (\gamma C_T + \epsilon_v) + 2(LD+G)^2\eta T + \frac{\epsilon_w + \gamma D}{\eta} + 4\sqrt{\frac{\gamma D T F V_T }{\eta}},
    \end{align*}
    where $\epsilon_v = D_{\mathcal{R}}(\theta_1, v_1),\epsilon_w = D_{\mathcal{R}}(\hat{w}_1,w_1^*)$.
\end{theorem}
\begin{proof}[Proof Sketch]
By Lemma~\ref{lem:noiseGradient}, the path-length version of instantaneous regret can be decomposed as 
\begin{align*}
     \mathbb{E} [l_t(w_t) - l_t(u_t)]  & \leq \mathbb{E} [\langle \nabla f(w_t;x_t,{y}_t), w_t - u_t \rangle  + 2(LD+G)\|w_t - \hat{w}_t\|]\\
    & = \mathbb{E} [ \underbrace{\langle \nabla f(w_t;x_t,{y}_t) - \nabla f(w_t;x_t,\hat{y}_t), w_t - \hat{w}_{t+1}\rangle}_{\text {term } \mathrm{A}} ]  + \mathbb{E} [ \underbrace{\langle \nabla f(w_t;x_t,\hat{y}_t) , w_{t} - \hat{w}_{t+1} \rangle}_{\text {term } \mathrm{B}}]  \\
    &\quad  + \mathbb{E} [\underbrace{\langle \nabla f(w_t;x_t,\tilde{y}_t), \hat{w}_{t+1} - u_t \rangle}_{\text {term } \mathrm{C}} ]   + \mathbb{E} [\underbrace{\langle \nabla f(w_t;x_t,{y}_t) -  \nabla f(w_t;x_t,\tilde{y}_t), \hat{w}_{t+1} - u_t \rangle}_{\text {term } \mathrm{D}} ] \\
    &\quad  + \mathbb{E}[2(LD+G)\|w_t - \hat{w}_t\|],
\end{align*}
where $u_1,\dots,u_T \in \mathcal{K}.$ Since the pseudolabel errors are bounded by Theorem~\ref{thm:labelError}, we know the bias $\|\hat{y}_t - {y}_t\| = 2 M_t.$ From the definition of algorithm, the bias $\|\tilde{y}_t - {y}_t\|$ is not larger than $\|\hat{y}_t - {y}_t\|$, we then have 
$\sum_{t=1}^T \|\tilde{y}_t - {y}_t\| \leq \sum_{t=1}^T \|\hat{y}_t - {y}_t\|\leq  \sum_{t=1}^T 2M_t.$
From this, as we assume the gradient and decision space are bounded, term A and term D can be bounded in terms of the errors bound, which is small and controllable by Lemma~\ref{thm:labelError}. Term C is bounded in terms of a recursive term by a proposition of mirror descent~\citep{beck2003mirror}. We prove that the implicit gradient update rule has a similar property with explicit gradient mirror descent, then bound term B in terms of a recursive term. By adding all the terms up and setting a suitable stepsize, we obtain a sum of recursion and get a path-length version of dynamic regret bound. Finally, using the similar technique with~\citep{besbes2015non,zhang2020simple}, we could transfer from path-length bound to temporal variability bound, i.e., bound the regret in terms of the continual domain shift.
\end{proof}
We then carefully choose the parameters to obtain the detailed result.
\begin{corollary}\label{cor:regret}
    In Theorem~\ref{thm:ExpectedRegretNoiseGradient}, set the stepsize $\eta = V_T^{1/3} T^{-1/3}.$ We have for $\sigma\leq R$
      \begin{align*}
        & \text{\rm D-Regret}^{\text{OSAMD}}(\{P_t\},T) \leq O(\sigma^{-2}C_T + V_T^{1/3} T^{2/3}).
      \end{align*}
\end{corollary}
Corollary~\ref{cor:regret} implies that the dynamic regret of OSAMD is controlled by active probability controller $\sigma,$ classifier shift $C_T,$ and continual domain shift $V_T$. Since we assume $C_T,V_T$ are constants, the regret is dominated by $O(T^{2/3}),$ i.e., $\text{D-Regret}(\{P_t\},T)\leq O(T^{2/3}),$ leading to the insight that OSAMD can compete with the optimal competitor in hindsight with fast convergence.

\paragraph{Remark} 
We do not provide the regret bound for $\sigma>R.$ Since in the worst case (e.g., all the samples are on the margin, then $|H(v_t)|=R$), the algorithm would query at least half of the labels in expectation, which is contrary to the assumption of limited labels.

We then begin to discuss the advantage of our design from a theoretical review.

\paragraph{The necessity of query}
We have the following lower bound for any unsupervised (i.e., without query) online algorithms.
\begin{theorem}\label{thm:NecessaryQuery}
    Assume $\alpha^* = 0.$ For any unsupervised online algorithm $\mathcal{A},$ there exists a sequence of $\{P_t\}$ satisfying our assumptions such that
    \begin{equation*}
        \text{\rm D-Regret}^{\mathcal{A}}(\{P_t\},T) \geq \Omega(T).
    \end{equation*}
\end{theorem}
We provide detailed proof in the appendix, where we create a special example. Theorem~\ref{thm:NecessaryQuery} implies that it is impossible for unsupervised algorithms to keep adaptive if the environment is changing over time, which shows the necessity of active query.

\paragraph{The efficacy of the online teacher-student structure}
It should be noted that the online teacher-student structure leads to tight regret bound, which is even aligned with the lower bound for online algorithms with full labels.
\begin{theorem}\label{thm:NecessarySelf}
    (\citet{besbes2015non}) Assume $\alpha^* = 0.$ Even if we have all the labels during the process, for any online learning algorithm $\mathcal{A}$, there exists a sequence of $\{P_t\}$ satisfying our assumptions such that
    \begin{align*}
        \text{\rm D-Regret}^{\mathcal{A}}(\{P_t\},T) \geq \Omega(V_T^{1/3}T^{2/3}).
    \end{align*}
\end{theorem}
This result is immediate from Theorem 2 of \citep{besbes2015non}, where OMD with suitable stepsize attains the lower bound. As illustrated in Theorem~\ref{thm:NecessarySelf}, all algorithms suffer from an $\Omega(V_T^{1/3}{T}^{2/3})$ dynamic regret lower bound. Recall our result of OSAMD is upper bounded by $O(V_T^{1/3}T^{2/3})$. From this, we conclude that the online teacher-student structure is effective in attaining tight regret bound, which shows that OSAMD can adapt well to the changing environments with an optimal convergence.

\paragraph{Remark} 
Note that the query strategy in Algorithm \ref{algorithm:OAST} is probabilistic, given by a Bernoulli distribution with probability $\frac{\sigma}{\sigma + |H_t (\theta_t)|}$, indicating that the query rate is relatively low when $\sigma$ is small. Notice that we also provide a regret bound for $\sigma\leq R$ that can be arbitrarily small in Corollary \ref{cor:regret}, which means the tight regret bound holds even in the limited-labels case, showing the advantage of the online teacher-student structure compared with previous works with full labels. 

\subsection{General Case}
In this subsection, we extend the results to the non-separable case, i.e., $\alpha^*>0$ but is a small or negligible constant (typical assumption in previous works~\citep{kumar2020understanding,wei2020theoretical}). We still begin with the analysis of the pseudolabel errors, as follows.
\begin{lemma}[Pseudolabel Errors]\label{thm:labelErrorGeneral}
  Under the same conditions in Lemma~\ref{thm:labelError}. Set 
  $
      \tau_t =  \min\{\frac{\sigma}{G^2},\frac{\max\{0,\sigma-y_t H_t(\theta_t)\}}{\| \nabla H_t(\theta_t)\|^2_*}\},\sigma \leq R.
  $
  The expected number of pseudolabel errors made by Algorithm~\ref{algorithm:OAST} is bounded by
  \begin{align*}
        & \mathbb{E}[\sum_{t=1}^T M_t]  \leq \frac{2 G^2}{\sigma^2}  (\gamma C_T + \epsilon_v + \frac{\sigma}{G^2} T \alpha^* ).
    \end{align*}
   where $M_t = \mathbf{1}_{\hat{y}_t \neq y_t}$ is the instantaneous mistake indicator, and $\epsilon_v = D_{\mathcal{R}}(\theta_1, v_1)$.
\end{lemma}
We provide detailed proof in the appendix, where we generalize the proof of the separable case. From Lemma~\ref{thm:labelErrorGeneral}, we observe that the expected pseudolabel errors are bounded by an $O(\alpha^* T )$ term, which is linear increasing. This cannot be eschewed, because any classifier would make mistakes if the data distribution is not separable. We then present the regret bound in such a case.
\begin{theorem}[Regret Bound]\label{thm:ExpectedRegretNoiseGradientGeneral}
  Under the same conditions and parameters in Lemma~\ref{thm:labelErrorGeneral}.  Algorithm~\ref{algorithm:OAST} achieves the following regret bound
    \begin{align*}
      &\text{\rm D-Regret}^{\text{OSAMD}}(\{P_t\},T)  \leq  \frac{\epsilon_w + \gamma D}{\eta} + 4\sqrt{\frac{\gamma D T F V_T }{\eta}}+ 2(LD+G)^2\eta T +  \frac{4(\eta G^4 + G^3 D)}{\sigma^2}  (\gamma C_T + \epsilon_v + \frac{\sigma}{G^2} T \alpha^* ) ,
    \end{align*}
    where $\epsilon_v = D_{\mathcal{R}}(\theta_1, v_1),\epsilon_w = D_{\mathcal{R}}(\hat{w}_1,w_1^*)$.
\end{theorem}
We provide detailed proof in the appendix, where it is a simple generation of the separable case. We next choose the parameters to obtain the detailed results.
\begin{corollary}\label{cor:regretgeneral}
    In Theorem~\ref{thm:ExpectedRegretNoiseGradientGeneral}, set the parameter $\eta=V_T^{1/3}T^{-1/3}$. We have for $\sigma\leq R$
      \begin{align*} 
        &\text{\rm D-Regret}^{\text{OSAMD}}(\{P_t\},T) \\
        &\leq O(\sigma^{-2}C_T+ V_T^{1/3} T^{2/3} + \sigma^{-1}{\alpha^*}T).
      \end{align*}
\end{corollary}

Since $C_T$ is assumed to be a constant. Therefore, from Corollary~\ref{cor:regretgeneral}, we show the regret is of $O(V_T^{1/3} T^{2/3} + {\alpha^*}T),$ suggesting that the OSAMD algorithm is comparable to the optimal in hindsight with only $O(\alpha^* T)$ bias, which is the best we can hope to achieve. Besides, recall the regret lower bound of traditional algorithms is $\Omega(V_T^{1/3}T^{2/3})$. As $\alpha^*$ is often much small and negligible, the convergent rate is still aligned with the theoretical lower bound. Thus, we can claim that OSAMD can still adapt well to the changing environment. 

\paragraph{Remark} 
The $\alpha^*$ bias can not be eschewed for any self-adaptive algorithms with limited labeled data. For instance, if all the data are on the margin of $v_t$, then the mistake probability is at least $\alpha^*/2R$. Since labeled data is limited, we could assume that number of unlabeled data is of $\Omega(T).$ Then the pseudolabel errors are of $\Omega(\alpha^* T/2R),$ which leads to $\Omega(\alpha^* T/2R)$ mistake feedbacks. Therefore, it is easy to prove that the dynamic regret suffers from an $\Omega(\alpha^* T/2R)$ term, which is aligned with the term 
of $O(\sigma^{-1}\alpha^* T)$ in Corollary~\ref{cor:regretgeneral}.

\section{Experiments}\label{sec:experiment}
In this section, we extensively evaluate OSAMD on both synthetic and real-world datasets. We first verify OSAMD on the synthesis dataset with continually changing distributions for the linear classification task. Then, we evaluate OSAMD on a deep learning model using a real-world dataset and demonstrate that the theoretical intuition can be applied to practical deep learning tasks as well. 

\begin{figure*}
		\centering
		\subfigure[Rotating Gaussian]{
    	\includegraphics[width=0.231\linewidth]{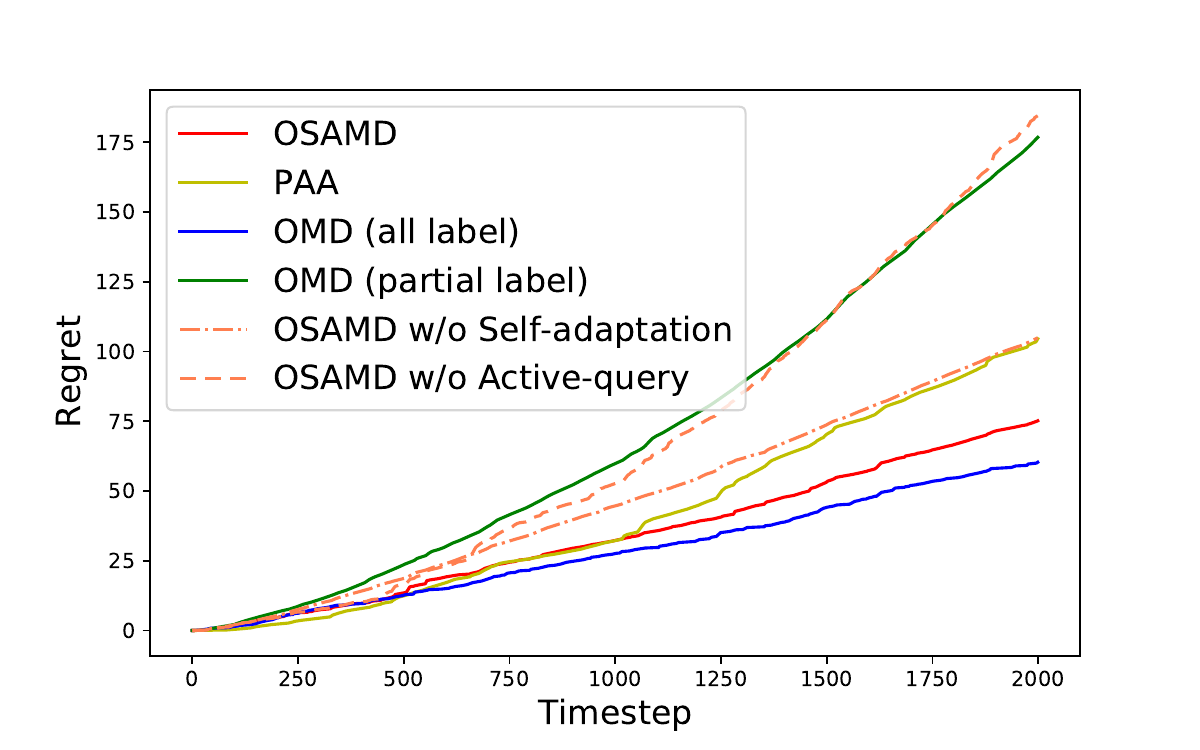}
    	\label{fig:LinearClassification}
    	}
    	\subfigure[Rotating MNIST]{
    	\includegraphics[width=0.231\linewidth]{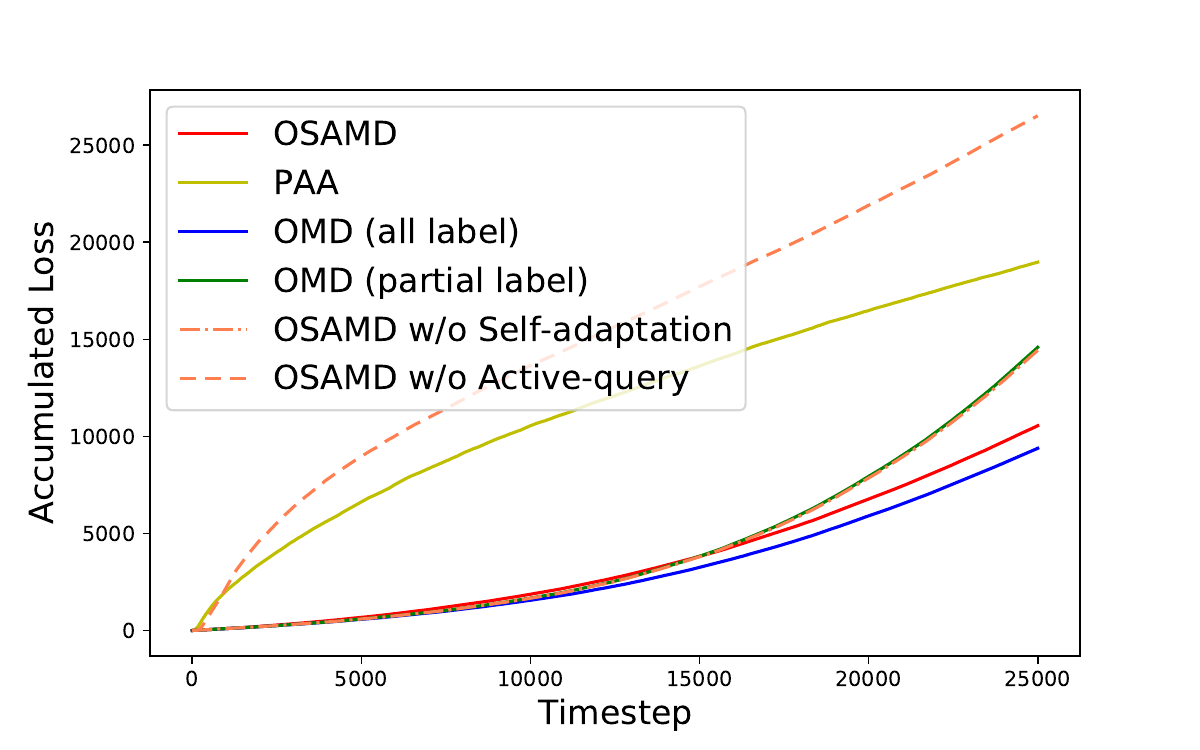}
    	\label{fig:MINIST}
    	}
    	\subfigure[Portraits]{
    	\includegraphics[width=0.231\linewidth]{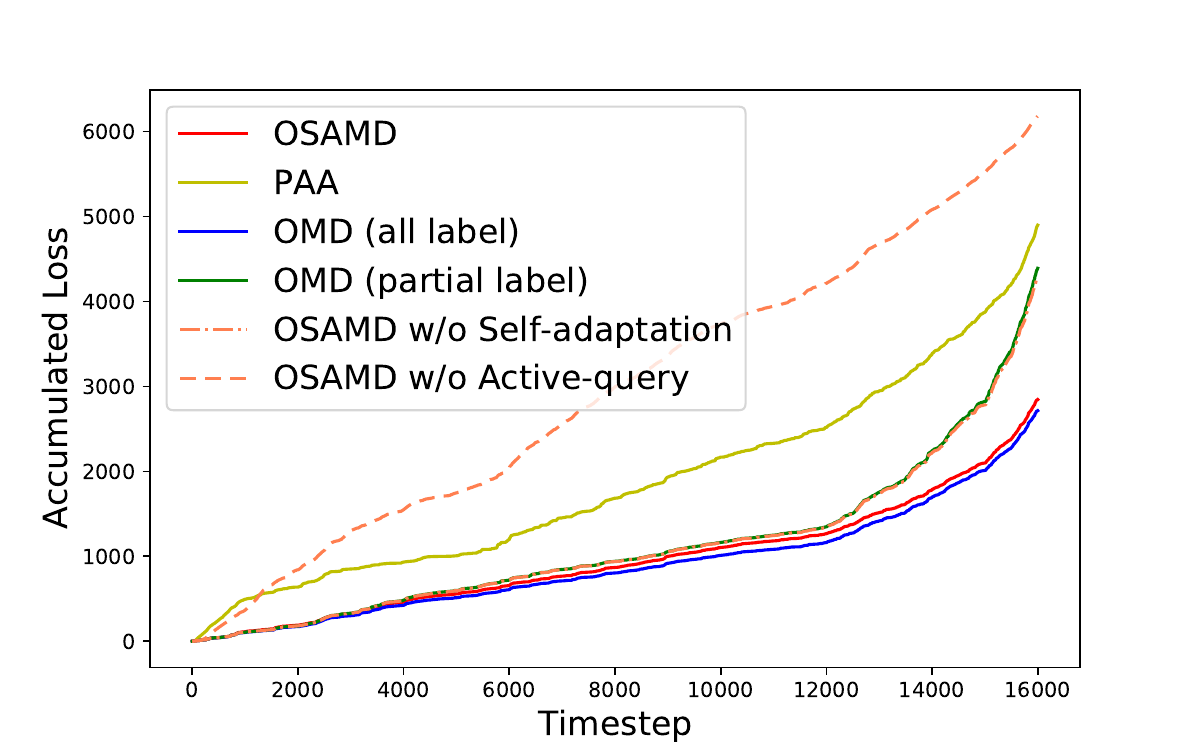}
    	\label{fig:Deep}
    	}
    	\subfigure[Cover-Type]{
    	\includegraphics[width=0.231\linewidth]{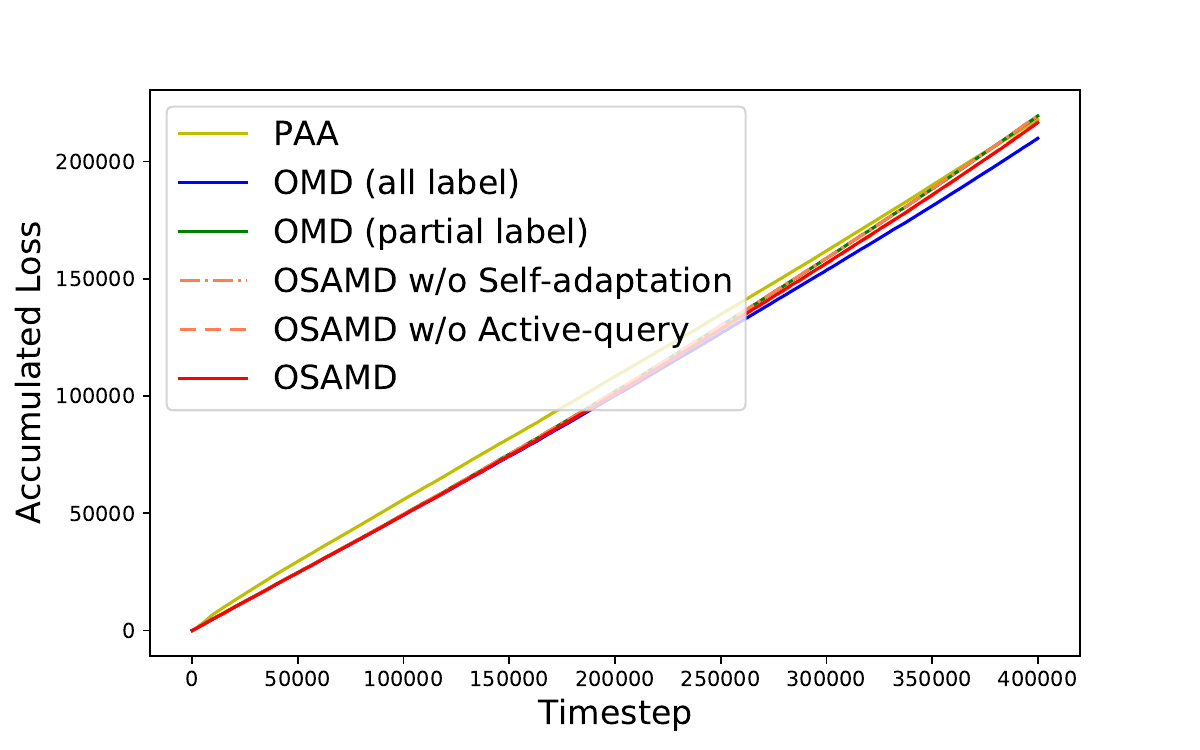}
    	\label{fig:cover}
    	}
    	\caption{Regret/Accumulated loss v.s. timestep on different datasets.}
    	\label{fig:experiment}
\end{figure*}
\begin{table*}
		\captionof{table}{Classification accuracy with 90$\%$ confidence intervals for the mean over $10$ runs.}
		\label{Tab:error}
		\centering
		\resizebox{0.99\linewidth}{!}{%
		\begin{tabular}{*9l}
        \toprule
        & \multicolumn{2}{c}{\textbf{Rotating Gaussian}}  & \multicolumn{2}{c}{\textbf{Rotating MNIST}}  & \multicolumn{2}{c}{\textbf{Portraits}}  & \multicolumn{2}{c}{\textbf{Cover-Type}} \\ 
        & Accuracy   & Labels    & Accuracy    & Labels & Accuracy   & Labels & Accuracy   & Labels\\ \midrule
    OSAMD  & \textbf{98.9}$\pm 0.2\%$  & 18.2$\pm 1.3\%$  & \textbf{86.8}$\pm 0.8\%$  & 6.4$\pm0.8\%$  &\textbf{93.7}$\pm 0.3\%$  & 3.8$\pm 0.8\%$   &  \textbf{76.8}$\pm 0.07\%$ & 2.3$\pm0.08\%$     \\ 
    PAA             & 98.5$\pm 0.2\%$           & 18.2$\pm 1.3\%$     & 84.0$\pm 0.6\%$& 6.4$\pm0.8\%$  & 92.0$\pm1.0$\% & 3.8$\pm 0.8\%$ & 76.5$\pm 0.02\%$& 2.3$\pm0.08\%$ \\
    OMD(all)        & \textbf{98.9}$\pm0.0\%$   & 100.0$\pm 0.0\%$      & \textbf{88.1}$\pm 0.7\%$      & 100.0$\pm0.0\%$   & \textbf{94.0}$\pm 0.4\%$  & 100.0$\pm 0.0\%$       & \textbf{76.8}  $\pm 0.02\%$           & 100.0$\pm 0.0\%$                          \\ 
    OMD(partial)    & 97.0$\pm1.0\%$            & 18.2$\pm 1.3\%$           & 81.7$\pm1.3 \%$      & 6.4$\pm0.8\%$ & 91.8$\pm1.1$\%   & 3.8$\pm 0.8\%$ & 75.3$\pm 0.04\%$   & 2.3$\pm0.08\%$                    \\ \midrule
    \begin{tabular}[c]{@{}l@{}}OSAMD w/o \\Self-adaption\end{tabular}   & 98.2$\pm 0.3\%$           & 18.2$\pm 1.3\%$         & 81.7$\pm 1.2\%$       & 6.4$\pm0.8\%$     & 91.8$\pm 1.1\%$  & 3.8$\pm 0.8\%$         & 75.3$\pm 0.03\%$   & 2.3$\pm0.08\%$        \\ 
    \begin{tabular}[c]{@{}l@{}}OSAMD w/o \\Active-query\end{tabular}   & 96.6$\pm 1.0\%$           & 18.2$\pm 1.3\%$  & 84.9$\pm 1.0\%$        & 6.4$\pm0.8\%$    & 92.7$\pm 0.6\%$  & 3.8$\pm 0.8\%$      &   76.1$\pm 0.16\%$                      & 2.3$\pm0.08\%$          \\ 
    \bottomrule
    \end{tabular}
    }
\end{table*}

\subsection{Experimental Setup}
We first briefly introduce our experimental setup, and leave the details in the appendix due to the space limit.

\paragraph{Dataset}
We experiment on two synthesis datasets - Rotating Gaussian (binary) \& Rotating MNIST (multi-class), and two real-world datasets - Portraits~\citep{ginosar2015century} \& Cover-Type~\citep{blackard1999comparative}: 
1) \textit{Rotating Gaussian}: 
We sequentially sample the data from two continually changing Gaussian distributions representing two classes. The center points rotate from $0\degree$ to $180\degree.$
2) \textit{Rotating MNIST}: We averagely rotate the images from $0\degree$ to $90\degree$ counterclockwise to be the target dataset with a continually changing domain.
3) \textit{Portraits}: It contains 37,921 photos of high school seniors labeled by gender. This real dataset suffers from a natural continual domain shift over the years~\citep{kumar2020understanding}. 
4) \textit{Cover-Type}: It contains 495141 samples with 54 features labeled by cover types. This dataset suffers from a natural continual domain shift over the sample indexes~\citep{kumar2020understanding}.

\paragraph{Baselines}
We compare with the following baselines:
1) \textit{PAA}~\citep{lu2016online}: To demonstrate the advantage of the online teacher-student structure, we compare with one online active algorithm -- passive-aggressive active (PAA) learning;
2) \textit{OMD (all)}: To compare OSAMD and traditional non-stationary online learning with full labels, we use all the labels to update by OMD;
3) \textit{OMD (partial)}: To compare OSAMD and OMD with the same amount of labeled samples, we use uniform sampled labels to update by OMD; 
4)\textit{OSAMD w/o Self-adaptation}: To evaluate the self-adaptation method of OSAMD, we use the same active queries as OSAMD to update by OMD;
5) \textit{OSAMD w/o Active-query}: To evaluate the active query strategy of OSAMD, we use uniform sampled labels to update the aggressive pseudolabel model for OSAMD. 

We do not compare with the proposed methods in \citet{kumar2020understanding} and \citet{chen2021active}, which study similar setting. For \citet{kumar2020understanding}, their algorithm should learn from batch data, thus can not be directly applied in the online continual adaptation setting that we study. For \citet{chen2021active}, the classification algorithm needs to memorize all the active queries during the online training and find the minimal of the received samples in every round. It is computationally intractable in our deep learning experiments. 

\paragraph{Implementation Details}
For Rotating Gaussian, we set the objective function to be the SVM loss. For other datasets, we follow the deep learning setting as previous work on unsupervised gradual domain adaptation~\citep{kumar2020understanding}. For rotating MNIST and Portraits, we used a 3-layer convolutional network with dropout(0.5) and batchnorm on the last layer. For the Cover-Type dataset we used a 2 hidden layer feedforward neural network with dropout(0.5). Due to space limitations, please refer to the supplementary material for more details.

\subsection{Experimental Results}
We next present the experimental results and ablation study for the proposed method.

\paragraph{Synthesis data}
We first investigate whether the simulation can corroborate our theoretical findings, and present the results illustrated in Table~\ref{Tab:error} (Rotating Gaussian, Rotating MNIST) and Figure~\ref{fig:LinearClassification},\ref{fig:MINIST}, from which we can make the following observations: \emph{OSAMD performs remarkably well with limited labels}. There is no accuracy decrease from the full-label OMD with only $18.2\%$ labels in Rotating Gaussian, and marginal decrease in Rotating MNIST with only $6.4\%$ labels. In the above two datasets, OSAMD achieves similar regret/accumulated loss with full-label OMD, showing remarkable adaptation ability. In contrast, the regrets/accumulated losses of other baselines increase dramatically. This experimental result is aligned with our dynamic regret bound in Theorem~\ref{thm:ExpectedRegretNoiseGradientGeneral}, where OSAMD has similar dynamic regret to full-label OMD with only a small bias in the general case.

\paragraph{Real-world data}
We then extend OSAMD to work with deep learning models, and observe the performance in practice. As shown in Table~\ref{Tab:error} (Portraits, Cover-Type) and Figure~\ref{fig:Deep},\ref{fig:cover}, the practical results are similar to the synthesis data. OSAMD attains $93.7\%$(Portraits)/$76.8\%$(Cover-Type) accuracy using only $3.8\%$/$2.3\%$ labels compared with $94.0\%$/$76.8\%$ accuracy of OMD (all) with full labels, while PAA and OMD (partial) with $3.8\%$/$2.3\%$ uniform queries only obtains $92.0\%$/$76.5\%$ and $91.8\%$/$75.3\%$. The accumulated loss of OSAMD is aligned with OMD (full), being a side-information to reflect the consistent of the regret, which demonstrates the remarkable adaptation ability to real-world environments. While the accumulated losses of other baselines increase quickly, showing the practical advantage of our theoretical design.

\paragraph{Ablation study}
Note that we have two key designs on OSAMD, i.e., self-adaptation and active query. To verify the efficacy of each component, we compare with two baselines: OSAMD w/o Self-adaptation and OSAMD w/o Active query. The experimental results in Table~\ref{Tab:error} and Figure~\ref{fig:experiment} show that: 1) \emph{The self-adaptation is effective}: OSAMD outperforms OSAMD w/o Self-adaptation, obtains a noticeable accuracy increase, and achieves a significant lower regret, which highlights the power of self-adaptation as in our theoretical findings. 2) \emph{The active query is effective}: OSAMD is more accurate and achieves a lower regret than OSAMD w/o Active-query, which demonstrates that the active queries are more effective than uniform samples. 

\paragraph{Sensitivity of Query Rate}
We test the performance regarding the effect of query number, as shown in Figure~\ref{fig:sensitivity}. We vary the parameter choice of $\sigma$ that balances the number of queries and performance, and compare it with other baselines on the Rotating Gaussian dataset. From the result, we find that the proposed OSAMD consistently outperforms others under different query rates, suggesting the advantage of OSAMD is not sensitive to the query number. 

\begin{figure}
    \centering
    \includegraphics[width=0.5\linewidth]{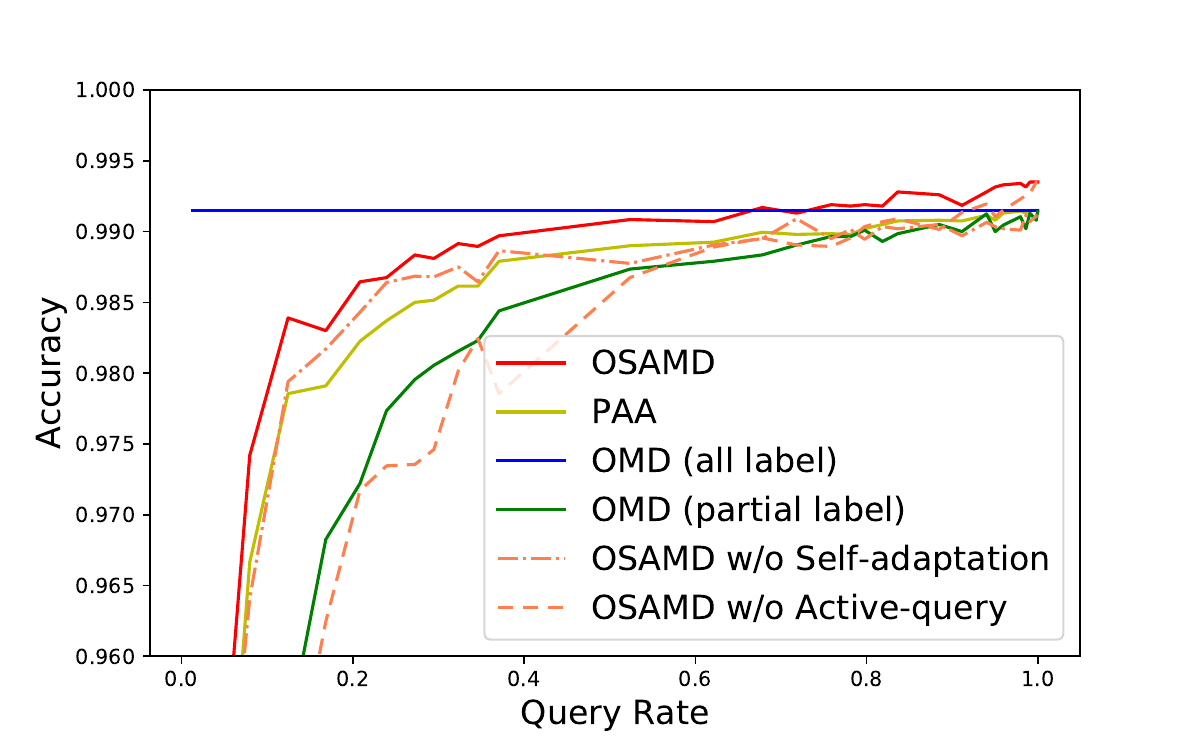}
    \caption{Accuracy v.s. query rate on Rotating Gaussian.}
    \label{fig:sensitivity}
\end{figure}

\section{Conclusions and Limitations}
This paper studies an open problem for machine learning models to continually adapt to changing environments with limited labels, where previous works show limitations on realistic modeling and theoretical guarantees. To fill this gap, we formulate the OACA problem and propose OSAMD, an effective online active learning algorithm with the novel design of the online teacher-student structure. We show it can compete with the optimal model in hindsight with optimal convergence order. Experimental evaluations corroborate our theoretical findings and verify the efficacy of OSAMD. Our results take the first step towards online domain adaptation in continually changing environments.

\paragraph{Limitations}
In this work, we tradeoff the number of queries and regret in an implicit way by setting the parameter $\sigma$. When seeking the explicit way, we face the common challenge of estimating the expected queries in online active learning literate~\citep{cesa2006worst,lu2016online}, and this becomes even harder under non-stationary settings. Although the limitation does not affect our experimental results and practical usage, it remains a future theoretical interest.

\subsubsection*{Acknowledgements}
This work is supported by the National Key Research and Development Program of China No. 2020AAA0106300, National Natural Science Foundation of China No. 62050110 and No. 61872215, and the Shenzhen Science and Technology Program of Grant No. RCYX20200714114523079. HZ would like to thank support from a Facebook research award. The authors also thank Kuaishou for sponsoring the research. We gratefully thank Tianyu Liu, Chenghao Hu, Jiangjing Yan, Yinan Mao for proof reading.

\bibliography{main}
\bibliographystyle{apalike}


\clearpage
\appendix

\thispagestyle{empty}


\section{FORMATTING INSTRUCTIONS FOR THE SUPPLEMENTARY MATERIAL}

We first provide a brief overview of the appendix. In Section~\ref{sec:related}, we introduce previous works related to our literature. In Section~\ref{supp:DA2TV}, we discuss how to connect various discrepancy measures between probability distributions in domain adaptation with the temporal variability condition in online learning, as described in Section~2 of the main paper. In Section~\ref{supp:proof}, we provide detailed proof for the analysis in Section~5 of the main paper. We then extend the results to the multiclass case in Section~\ref{supp:multiclass}. We provide the experimental details in Section~\ref{supp:additionalexperiment}.

\section{Related Work}
\label{sec:related}
The topic of this paper sits well in between two amazing bodies of literature: domain adaptation~\citep{tzeng2014deep,ganin2015unsupervised,hoffman2018cycada,zhao2020review} that is a typical method to improve the generalization of a pre-trained model when testing on new domains without or with limited labels, and online learning~\citep{hazan2016introduction} that is a basic framework for learning with streaming online data. Our results therefore contribute to both fields and hopefully will inspire more interplay between the two communities.

\subsection{Domain Adaptation}
In the domain adaptation literature, our setting is related to active domain adaptation that queries additional labels to enable effective adaptation, and gradual domain adaptation that studies the adaptation problem under gradual domain shift. We present a brief summary as follows.

\paragraph{Active Domain Adaptation}~~
Active Domain Adaptation~\citep{rai2010domain,chattopadhyay2013joint,su2020active,prabhu2021active} aims to actively select the most representative samples from the target domain, and learn a model to maximize performance on the target set. It was first proposed by \citet{rai2010domain} with application to sentiment classification from text data, where they embedded an online uncertainty-based sample strategy in domain adaptation. \citet{chattopadhyay2013joint} proposed a method that performs transfer and active learning simultaneously by solving a single convex optimization problem. Recently, active adversarial domain adaptation (AADA)~\citep{su2020active} was proposed to solve the active domain adaptation problem in the context of deep learning, where AADA selects samples based on the uncertainty measured by entropy and targetness measured by the domain discriminator. \citet{prabhu2021active} proposed ADA-CLUE that queried labels based on uncertainty and diversity, then adopts a semi-supervised domain adaptation to transfer the domain knowledge to the target. However, current works are designed to adapt from a fixed source domain to a fixed target domain, and can not be applied to continual domain adaptation in the changing environment.

\paragraph{Gradual Domain Adaptation}~~
Gradual domain adaptation~\citep{hoffman2014continuous,gadermayr2018gradual,wulfmeier2018incremental,bobu2018adapting,kumar2020understanding} cares about how to adapt the model to a changing environment with unlabeled data. Continuous manifold learning~\citep{hoffman2014continuous} tried to adapt to evolving visual domains by learning a sequence of transformations on a fixed source representation. \citet{gadermayr2018gradual} extended previous approaches by adding two methods for regularization of the fully-unsupervised adaptation process. \citet{wulfmeier2018incremental} presented an adversarial approach benefiting from unsupervised alignment to a series of intermediate domains. \citet{bobu2018adapting} proposed a continuous replay model that enforced the past prediction to be matched. \citet{kumar2020understanding} first developed a theory, and proposed a gradual self-training method, which self-trains on the finite unlabeled examples from each batch successively. However, the generalization bound in~\citep{kumar2020understanding} suffers from an exponential error blow-up in time horizon $T$. Our analysis further shows that unsupervised methods suffer from linear regret even in the separable case, implying the necessity of additional labels in the dynamic online setting.

\subsection{Online Learning}
In the online learning literature, our setting is related to adaptive online learning that aims to achieve optimal bound in dynamic environments, online active learning that studies online classification with active queries, online meta learning that provides a framework for adapting to a new domain with few-shot samples. We present a brief summary as follows.

\textbf{Adaptive Online Learning}~~
Adaptive online learning~\citep{besbes2015non,mokhtari2016online,jadbabaie2015online} extends the traditional online learning setting to deal with dynamic problems, by introducing the dynamic regret that measures the online performance in dynamic environments. Under the path-length~\citep{zinkevich2003online,hall2013dynamical} or temporal variability~\citep{besbes2015non,campolongo2020temporal} conditions, sublinear regret is achieved by online algorithms with suitable stepsizes~\citep{yang2016tracking}. However, practical deployments of fully online learning systems have been somewhat limited and impractical, partly due to the expense of annotations. 

\textbf{Online Active Learning}~~
Previous works~\citep{cesa2005minimizing,lu2016online,hao2017second} study online active learning for classification. However, online classification with limited labels in changing environments remains an open question~\citep{shuji2017budget}. Recent work~\citep{chen2021active} considers online active learning with hidden covariate shift for regression tasks. However, both the algorithm and theory can not be generalized to online classification with joint distribution shift. In this paper, we tackle this problem and resolve the open question proposed in~\citep{lu2016online,shuji2017budget}.

\textbf{Online Meta Learning}~~
Online meta learning~\citep{finn2019online,balcan2019provable,khodak2019adaptive} provides a framework for online few shot adaptation by learning the meta regularization. It studies how the model can fast adapt to a new environment using only a few samples by capturing the optimal initialization. However, online meta-learning focuses on ``few-samples learning’’ using passively received labeled samples (usually not sufficient to achieve sublinear regret), while our setting focuses on ``few-labels learning’’, where the active queries and unlabeled samples also help the adaptation.

\section{Domain Discrepancy to Temporal Variability}\label{supp:DA2TV}
In this section, we discuss how to connect classic distance metrics between probability distributions to the temporal variability condition used in the online learning literature. We present all the results in Table~\ref{Tab:Assumption}, where we provide the conditions for connecting these two.

\begin{table}[h]
\caption{The conditions for connecting different domain discrepancy with the temporal variability condition}
\centering
\label{Tab:Assumption}
\resizebox{0.99\columnwidth}{!}{%
\begin{tabular}{*2l}
\toprule
             &  Temporal Variability      \\ \midrule
Bounded sum of Total Variation       & Bounded function $f$ \\ 
Bounded sum of Wasserstein Infinity  & $f$ is Lischitz continuous on $x$; No label shift    \\ 
Bounded sum of Maximum  Mean Discrepancy & Bounded reproducing kernel Hilbert space $\mathcal{K}$; Linear function $f$ \\ \bottomrule
\end{tabular}%
}
\end{table}

\subsection{Bounded Sum of Total Variation}
We first show that the bounded sum of total variation~\citep{ben2010theory,zhao2018multiple,zhao2019learning} (as Assumption~1) with bounded function (as Assumption~5) can lead to the temporal variability condition in the online learning literature.
\begin{prop}\label{prop:TotalVariation}
    Assume the sum of total variation between $P_t,P_{t+1}$ is bounded, i.e., satisfying Assumption~1. If the function value $f(w;x,y)$ is bounded for all $w\in \mathcal{K},x\in \mathcal{X},y\in\{-1,1\},$ i.e., satisfying Assumption~5. Then the temporal variability is bounded as following
    \begin{align*}
        \sum_{t=1}^{T-1} \sup_{w\in \mathcal{K}}|l_t(w) - l_{t+1}(w) | \leq 2F V_T.
    \end{align*}
\end{prop}
\begin{proof}
    First, by the definition of $l_t$ and bounded $f,$ we have
    \begin{align*}
        \sup_{w\in \mathcal{K}}\left|l_t(w) - l_{t+1}(w)\right| & = \sup_{w\in \mathcal{K}}\left|\mathbb{E}_{x,y \sim P_t (x,y)} [f(w;x,y)] - \mathbb{E}_{x,y \sim P_{t+1} (x,y)} [f(w;x,y)]\right|\\
        & \leq \sup_{w\in \mathcal{K}}\int_{x,y} \left|f(w;x,y) d P_t  -  f(w;x,y) d P_{t+1} \right| \\
        & \leq \sup_{w\in \mathcal{K}}\int_{x,y} \left|f(w;x,y)| | d P_t - d P_{t+1}\right| \\
        & \leq F \int_{x,y} \left|d P_t - d P_{t+1}\right|.
    \end{align*}
    The last inequality is from Assumption~5. Then, sum this term from $1$ to $T-1,$ and by the definition of $d_{TV}$, we obtain
    \begin{align*}
        \sum_{t=1}^{T-1} \sup_{w\in \mathcal{K}}\left|l_t(w) - l_{t+1}(w) \right| & \leq F \sum_{t=1}^{T-1} \int_{x,y} \left|d P_t  - d P_{t+1} \right|\\
        & \leq F\cdot 2 \sum_{t=1}^{T-1} \sup_{E}\left|P_t(E) - P_{t+1}(E)\right| \\
        & = 2F \sum_{t=1}^{T-1} \dtv (P_t,P_{t+1})\\
        & \leq 2F V_T.
    \end{align*}
    The last inequality is from Assumption~1. We thus end the proof.
\end{proof}

The Proposition~\ref{prop:TotalVariation} also holds for the multiclass case, since we define the total variation by measurable events, which do not depend on the class set.

\subsection{Bounded Sum of Wasserstein Infinity Distance}
We next present the bounded sum of wasserstein infinity distance that also leads to the temporal variability condition in the online learning literature, under conditions that $f$ is Lischitz continuous over $x$ and $P_t$ has no label shift. We begin with the definition of wasserstein infinity distance.
\begin{define}[Wasserstein Infinity Distance]
We use $W_{\infty}(\distp, \distq)$ to denote the Wasserstein-infinity distance between distributions $\distp$ and $\distq$:
    \begin{align*} 
        W_{\infty}(\distp, \distq)\defeq\inf \left\{ \sup _{x \in \mathcal{X}}\|h(x)-x\|: h: \mathcal{X} \rightarrow \mathcal{X}, h_{\#} \distp=\distq\right\} ,
    \end{align*}
where $\#$ denotes the push-forward of a measure, that is, for every set $A \subseteq \mathcal{X}, h_{\#} P(A)=$ $P\left(h^{-1}(A)\right).$
\end{define}
\textbf{Remark}~~Note that in the definition above we use the Monge formulation of the Wasserstein distance. Under mild assumptions, e.g., both $\distp$ and $\distq$ have densities, the Monge formulation is well-defined. This formulation has also been used in a previous work~\citep{kumar2020understanding} to measure the distributional shift.

In particular, the authors~\citep{kumar2020understanding} assume that the conditional distributions do not shift too much, i.e.,
\begin{align*}
        \rho(\distp, \distq)\defeq\max \left(W_{\infty}\left(\distp_{X \mid Y=1}, \distq_{X \mid Y=1}\right),W_{\infty}\left(\distp_{X \mid Y=-1}, \distq_{X \mid Y=-1}\right)\right)
\end{align*}
is bounded, and there is no label shift, i.e., $\distp(Y)= \distq(Y).$ We adopt similar assumptions and further assume that $f(w;x,y)$ is Lipschitz continuous on $x,$ a general assumption on the loss function, which leads to temporal variability:
\begin{prop}\label{prop:Wasserstein}
    Let the function value $f(w;x,y)$ be Lipschitz continuous on $x,$ i.e., there exists a constant $L \geq 0$, such that
    \begin{align*}
        |f(w;x_1,y) - f(w;x_2,y)| \leq L \|x_1 - x_2\|, \forall {x_1,x_2\in\mathcal{X}},w\in\mathcal{K}.
    \end{align*}
    Assume the sum of Wasserstein Infinity Distance between each consequent pair of conditional distribution is bounded, i.e.
    \begin{align*}
        \sum_{t=1}^{T-1} \rho({P_{t}}, {P_{t+1}}) \leq V_T .
    \end{align*}
    Further assume there is no label shift, i.e., $\forall t\in[T], P_{t}(Y)=P_{t+1}(Y)=P(Y).$ 
    Then the temporal variability is bounded, as follows
    \begin{align*}
        \sum_{t=1}^{T-1} \sup_{w\in \mathcal{K}}|l_t(w) - l_{t+1}(w) | \leq L V_T.
    \end{align*}
\end{prop}
\begin{proof}
    First, by the definition of Wasserstein Infinity Distance, we know that there exist $h_t^{(y)},~y\in\{-1,1\}$ such that 
    \begin{equation*}
        {h_t^{(y)}}_{\#} P_t(\cdot\mid y) = P_{t+1}(\cdot \mid y)
    \end{equation*}
    and 
    \begin{equation*}
        \sup _{x \in \mathcal{X}}\left\|h_t^{(y)}(x)-x\right\| \leq \rho({P_{t}}, {P_{t+1}})+\epsilon,\quad\forall \varepsilon >0.
    \end{equation*}
    Then, by the definition of $l_t,$ we have
    \begin{align*}
        &\left|l_t(w) - l_{t+1}(w)\right| \leq \sum_{y=-1,1}\left|P(Y = y)\mathbb{E}_{x \sim P_t (x\mid y)} [f(w;x,y)] - P(Y = y)\mathbb{E}_{x \sim P_{t+1}(x\mid y)} [f(w;x,y)]\right| \\
        & = \sum_{y=-1,1} \left| \int_{x} f(w;x,y) P (y) d P_t (X|Y=y) - \int_{x} f(w;x,y) P(y) d P_{t+1} (X|Y=y)\right| \\
        & = \sum_{y=-1,1}\left| \int_{x} f(w;x,y)P (y)   d P_t (X|Y=y) - \int_{x} f(w;h_t^{(y)}(x),y)P(y) d P_t (X|Y=y)\right| \\
        & \leq \sum_{y=-1,1} \int_{x} \left| f(w;x,y) P (y) -   f(w;h_t^{(y)}(x),y) P (y)\right| d P_t (X|Y=y) \\
        & \leq L \int_{x} \|x-   h_t^{(1)}(x)\| P (Y=1) d P_t (X|Y=1) + L \int_{x} \|x-   h_t^{(-1)}(x)\| P (Y=-1) d P_t (X|Y=-1) \\
        & \leq L \max_{h=h_t^{(1)},h_t^{(-1)}}\sup _{x \in \mathcal{X}}\|h(x)-x\|_{2} \left(\int_{x} P (Y = 1) d P_t (X|Y=1) + \int_{x} P (Y = -1) d P_t (X|Y=-1)\right)\\
        & \leq L (\rho({P_{t}}, {P_{t+1}}) + \varepsilon), \forall \varepsilon >0.
    \end{align*}
    There is no $w$ in right side, and thus the inequality still hold if we take $\sup_w$ in the left side. Then 
    \begin{align*}
        \sup_{w\in \mathcal{K}}\left|l_t(w) - l_{t+1}(w)\right| \leq \inf_{\varepsilon >0} L (\rho({P_{t}}, {P_{t+1}}) + \varepsilon) = L \rho({P_{t}}, {P_{t+1}}).
    \end{align*}
    By summing up, we get
    \begin{align*}
        \sum_{t=1}^{T-1} \sup_{w\in \mathcal{K}}\left|l_t(w) - l_{t+1}(w) \right| & \leq L \sum_{t=1}^{T-1} \rho({P_{t}}, {P_{t+1}}) \leq L V_T.
    \end{align*}
\end{proof}

\subsection{Bounded Sum of Maximum Mean Discrepancy}
We finally show that under the conditions that the decision space $\mathcal{K}$ is a bounded reproducing kernel Hilbert space and $f$ is linear on the representation space, the bounded sum of maximum mean discrepancy~\citep{long2015learning} can lead to the temporal variability condition in the online learning.
\begin{define}[Maximum Mean Discrepancy]
We use ${\text{MMD}}(\distp, \distq)$ to denote the maximum mean discrepancy between distributions $\distp$ and $\distq$:
    \begin{equation*}
        {\text{MMD}}_{\phi}(\distp, \distq)\defeq \|\mathbb{E}_{x\sim\distp}[\phi(x)] - \mathbb{E}_{x\sim\distq}[\phi(x)]\|_{\mathcal{H}},
    \end{equation*}
    where feature map $\phi:\mathcal{X}\rightarrow \mathcal{H},$ and $\mathcal{H}$ is a reproducing kernel Hilbert space. In binary class, the distance between conditional distribution
    \begin{equation*}
        d_{\text{MMD}}^\phi(\distp, \distq)\defeq \max \{{\text{MMD}}_{\phi}\left(\distp_{X \mid Y=1}, \distq_{X \mid Y=1}\right),{\text{MMD}}_{\phi}\left(\distp_{X \mid Y=-1}, \distq_{X \mid Y=-1}\right)\}.
    \end{equation*}
\end{define}

\begin{prop}\label{prop:MMD}
    Let $\mathcal{K}$ to be a reproducing kernel Hilbert space. Assume the sum of Maximum Mean Discrepancy between conditional $P_t,P_{t+1}$ is bounded, i.e.
    \begin{align*}
        \sum_{t=1}^{T-1} d_{\text{MMD}}^{\phi}(P_{t}, P_{t+1}) \leq V_T,
    \end{align*}
    where $\phi:\mathcal{X}\rightarrow \mathcal{K}.$
    Let $f(w;x,y) = y \langle w, \phi(x)\rangle$ linear on the representation space. Assume $\mathcal{K}$ is bounded by $\|w\|_{\mathcal{H}}\leq F, \forall w\in \mathcal{K},$ then the temporal variability is bounded, as following
    \begin{align*}
        \sum_{t=1}^{T-1} \sup_{w\in \mathcal{K}}|l_t(w) - l_{t+1}(w) | \leq F V_T.
    \end{align*}
\end{prop}
\begin{proof}
    From the linear property of $f,$ and the definition of $l_t,$ we have
    \begin{align*}
        l_t(w) = \sum_{y=-1,1} P(Y = y)\mathbb{E}_{x \sim P_t (x\mid y)} f(w;x,y) = \sum_{y=-1,1} P(Y = y) y \langle w, \mathbb{E}_{x \sim P_t (x\mid y)} \phi_w(x)\rangle.
    \end{align*}
    Then, by the definition of Maximum Mean Discrepancy
    \begin{align*}
        & \sup_{w\in \mathcal{K}} \left|l_t(w) - l_{t+1}(w)\right| \\
        & =  \sup_{w\in \mathcal{K}} |\sum_{y=-1,1}P(Y = y) y \langle w, \mathbb{E}_{x \sim P_t (x\mid y)} \phi(x)\rangle - \sum_{y=-1,1} P(Y = y) y \langle w, \mathbb{E}_{x \sim P_{t+1} (x\mid y)} \phi(x)\rangle|\\
        & \leq \sup_{w\in \mathcal{K}} \sum_{y=-1,1} P(Y = y) \|w\|_{\mathcal{H}} \|\mathbb{E}_{x \sim P_t (x\mid y)} \phi(x) - \mathbb{E}_{x \sim P_{t+1} (x\mid y)} \phi(x)\|_{\mathcal{H}}\\
        & \leq F d_{\text{MMD}}^\phi(P_{t}, P_{t+1}).
    \end{align*}
    The first inequality comes from the H\"older inequality. By summing up, we finally get
    \begin{align*}
        \sum_{t=1}^{T-1} \sup_{w\in \mathcal{K}}\left|l_t(w) - l_{t+1}(w) \right|  \leq  \sum_{t=1}^{T-1} F d_{\text{MMD}}^\phi(P_{t}, P_{t+1}) \leq F V_T.
    \end{align*}
\end{proof}

\section{Missing Proofs}\label{supp:proof}

In this section, we provide the detailed proof of the pseudolabel errors bound and the dynamic regret bound for OSAMD.

\subsection{Pseudolabel Errors}
In this subsection, we analyze the pseudolabel errors for the OSAMD algorithm. We will first present some useful lemmas, then provide the proof of the separable case, where the data distribution can be correctly classified within a margin (i.e., $\alpha^* = 0$ in Assumption~2). Finally, we generalize it to the non-separable case. Here we generalize the proof in \citet{lu2016online,cesa2006worst} to non-stationary cases with mirror descent.

We first introduce the lemma on the property of Bregman divergence.
\begin{lemma}[\citet{beck2003mirror}]\label{lem:BeckandTeboulle}
    Let $\mathcal{K}$ be a convex set in a Banach space $\mathcal{B},$ and regularizer $\mathcal{R}: \mathcal{K} \mapsto \mathbb{R}$ be a convex function, and let $D_{\mathcal{R}}(\cdot, \cdot)$ be the Bregman divergence induced by $\mathcal{R}$. Then, any update of the form
    \begin{equation*}
        w^{*}=\underset{w \in \mathcal{K}}{\arg \min }\left\{\langle a, w\rangle+D_{\mathcal{R}}(w, c)\right\}
    \end{equation*}
    satisfies the following inequality
    \begin{equation*}
        \left\langle w^{*}-d, a\right\rangle \leq D_{\mathcal{R}}(d, c)-D_{\mathcal{R}}\left(d, w^{*}\right)-D_{\mathcal{R}}\left(w^{*}, c\right)
    \end{equation*}
    for any $d \in \mathcal{K}.$
\end{lemma}

Denote the instantaneous hinge loss with margin $r$ by $f^r_t(\theta) = \max\{0, r - y_t H(\theta;x_t) \},$ where $x_t,y_t$ is sampled from $P_t.$ We then present a useful lemma to get the recurrence.
\begin{lemma}\label{lem:firstInequation}
    When regularizer $\mathcal{R}: \mathcal{K} \mapsto \mathbb{R}$ is a 1-strongly convex function on $\mathcal{K}$ with respect to a norm $\|\cdot\|$. Then for algorithm~1, the following inequality holds
    \begin{align*}
        &\tau_t r  - \tau_t y_t H_t(\theta_t) - \frac{\tau_t^2}{2}\|\nabla H_t(\theta_t)\|^2_* \leq D_{\mathcal{R}} (v_t,\theta_t) - D_{\mathcal{R}} (v_t,\theta_{t+1}) + \tau_t f^r_t(v_t) .
    \end{align*}
    for $r > 0.$
\end{lemma}
\begin{proof}
    First, by the definition of $f^r_t,$ we have
    \begin{align*}
        r -  f^r_t(v_t)  &  = r - \max\{0,r-y_t H_t(v_t)\} \leq  y_t H_t(v_t) \\
        & = y_t H_t(v_t) - y_t H_t(\theta_t)  + y_t H_t(\theta_t).  
    \end{align*}
    By the convexity of $-y H(\cdot),$ we have
    \begin{align*}
        y_t H_t(v_t)  - y_t H_t(\theta_t) & = - y_t H_t(\theta_t) - ( - y_t H_t(v_t) )\\
        & \leq \langle - y_t \nabla H_t(\theta_t), \theta_t - v_t \rangle \\
        & \leq \langle - y_t \nabla H_t(\theta_t), \theta_{t+1} - v_t \rangle + \langle - y_t \nabla H_t(\theta_t), \theta_{t} - \theta_{t+1} \rangle.
    \end{align*}
    By the update rule of $\theta$ and Lemma~\ref{lem:BeckandTeboulle}, the first term can be bounded by
    \begin{align*}
        &\langle - y_t \nabla H_t(\theta_t), \theta_{t+1} - v_t \rangle \leq \frac{1}{\tau_t} (D_{\mathcal{R}}(v_t,\theta_t) - D_{\mathcal{R}}(v_t,\theta_{t+1}) - D_{\mathcal{R}} (\theta_{t+1},\theta_{t})  ).
    \end{align*}
    Due to H\"older inequality and the fact that $ab\leq \frac{\eta}{2} a^2 + \frac{1}{2\eta}G^2$ for $\eta >0,$ for the second term, we have
    \begin{align*}
        \langle - y_t \nabla H_t(\theta_t), \theta_{t} - \theta_{t+1} \rangle &  \leq   \|\nabla H_t(\theta_t)\|_* \|\theta_{t+1} - \theta_t\|\\
        & \leq \frac{\tau_t}{2} \| \nabla H_t(\theta_t)\|_*^2 + \frac{1}{2\tau_t} \|\theta_{t+1} - \theta_t\|^2.
    \end{align*}
    Due to the strong convexity of regularizer $\mathcal{R},$ we have $D_{\mathcal{R}}(x, y) \geq \frac{1}{2}\|x-y\|^{2}$ for any $x, y \in \mathcal{X}$ \citep{mohri2018foundations}. Therefore, by plugging the above term, we obtain that
    \begin{align*}
        r -  f^r_t(v_t)  & \leq \frac{1}{\tau_t} (D_{\mathcal{R}}(v_t,\theta_t) - D_{\mathcal{R}}(v_t,\theta_{t+1}) -\frac{1}{2} \|\theta_{t+1} - \theta_t\|^2) \\
        & \quad + \frac{\tau_t}{2} \| \nabla H_t(\theta_t)\|_*^2 + \frac{1}{2\tau_t} \|\theta_{t+1} - \theta_t\|^2 +  y_t H_t(\theta_t).
    \end{align*}
    By rearranging, we have
    \begin{align*}
        &\tau_t r  - \tau_t y_t H_t(\theta_t) - \frac{\tau_t^2}{2}\|\nabla H_t(\theta_t)\|^2_* \leq D_{\mathcal{R}} (v_t,\theta_t) - D_{\mathcal{R}} (v_t,\theta_{t+1}) + \tau_t f^r_t(v_t) .
    \end{align*}
\end{proof}

Denote the instantaneous mistake by $M_t (w) = \mathbf{1}_{\hat{y}_t \neq y_t},$ and let $L_t (w) = \mathbf{1}_{\hat{y}_t = y_t, \hat{y}_t H_t(w)\leq r}$ to be the indicator of the right decision but in the margin $r$, where $\mathbf{1}_{(\cdot)}$ is the indicator function. We then have the following relationship.
\begin{lemma}\label{lem:instantaneousError}
    Take the same assumptions as Lemma~\ref{lem:firstInequation}. For Algorithm~1, let $\tau_t = 0$ if $ {f^r_t(\theta_t)}=0,$ then the following inequality holds for every $t$
    \begin{align*}
        & M_t Z_t \tau_t(r + |H_t(\theta_t)| - \frac{\tau_t}{2}\|\nabla H_t(\theta_t)\|^2_*)  + L_t Z_t \tau_t(r - |H_t(\theta_t)| - \frac{\tau_t}{2}\|\nabla H_t(\theta_t)\|^2_*) \\
        & \leq D_{\mathcal{R}}(v_t,\theta_t)-D_{\mathcal{R}}(v_t,\theta_{t+1}) + \tau_t f^r_t(v_t) 
    \end{align*}
    for $r>0.$
\end{lemma}
\begin{proof}
    From Lemma~\ref{lem:firstInequation}, we know that
    \begin{align*}
        &\tau_t r  - \tau_t y_t H_t(\theta_t) - \frac{\tau_t^2}{2}\|\nabla H_t(\theta_t)\|^2_* \leq D_{\mathcal{R}} (v_t,\theta_t) - D_{\mathcal{R}} (v_t,\theta_{t+1}) + \tau_t f^r_t(v_t) .
    \end{align*}
    If $M_t=1$ then $y_t H_t(\theta_t)\leq 0,$ and if $L_t = 1$ then $y_t H_t(\theta_t)\geq 0.$ Therefore, we can obtain
    \begin{align*}
        & M_t Z_t \tau_t(r + |H_t(\theta_t)| - \frac{\tau_t}{2}\|\nabla H_t(\theta_t)\|^2_*)  + L_t Z_t \tau_t(r - |H_t(\theta_t)| - \frac{\tau_t}{2}\|\nabla H_t(\theta_t)\|^2_*) \\
        & \leq M_t Z_t (D_{\mathcal{R}}(v_t,\theta_t)-D_{\mathcal{R}}(v_t,\theta_{t+1}) + \tau_t f^r_t(v_t)) + L_t Z_t (D_{\mathcal{R}}(v_t,\theta_t)-D_{\mathcal{R}}(v_t,\theta_{t+1}) + \tau_t f^r_t(v_t))\\
        & = (M_t + L_t) Z_t (D_{\mathcal{R}}(v_t,\theta_t)-D_{\mathcal{R}}(v_t,\theta_{t+1}) + \tau_t f^r_t(v_t)).
    \end{align*}
    From Algorithm~1, we know that if $Z_t=0,$ then $\tau_t = 0, \theta_t = \theta_{t+1}.$ And if $M_t + L_t =0,$ we get $y_t H_t(\theta_t) \geq r,$ then $\tau_t=0, \theta_t = \theta_{t+1}.$ Therefore, we have
    \begin{align*}
        (M_t + L_t) Z_t (D_{\mathcal{R}}(v_t,\theta_t)-D_{\mathcal{R}}(v_t,\theta_{t+1}) + \tau_t f^r_t(v_t)) = D_{\mathcal{R}}(v_t,\theta_t)-D_{\mathcal{R}}(v_t,\theta_{t+1}) + \tau_t f^r_t(v_t).
    \end{align*}
    We finally get
    \begin{align*}
        & M_t Z_t \tau_t(r + |H_t(\theta_t)| - \frac{\tau_t}{2}\|\nabla H_t(\theta_t)\|^2_*)  + L_t Z_t \tau_t(r - |H_t(\theta_t)| - \frac{\tau_t}{2}\|\nabla H_t(\theta_t)\|^2_*) \\
        & \leq D_{\mathcal{R}}(v_t,\theta_t)-D_{\mathcal{R}}(v_t,\theta_{t+1}) + \tau_t f^r_t(v_t). 
    \end{align*}
\end{proof}

\subsubsection{Separable Case}
Here, we analyze pseudolabel errors for the separable case, i.e., $\alpha^* = 0,$ where we can easily know that $f^r_t(v_t) = 0$ if $r\leq R.$ Before proving the theorem, we first present the following lemma.
\begin{lemma}\label{lem:PA1InstantaneousError}
    Take the same assumptions as Lemma~1. Let $\tau_t = f^r_t (\theta_t) / \| \nabla H_t(\theta_t)\|^2_*.$ Then for Algorithm~1, the following inequality holds
    \begin{align*}
       \frac{r}{2G^2}M_t Z_t (r + |H_t(\theta_t)| ) \leq D_{\mathcal{R}} (v_t,\theta_t) - D_{\mathcal{R}} (v_t,\theta_{t+1}),
    \end{align*}
    for $r\leq R.$
\end{lemma}
\begin{proof}
    By the separability, we know $f^r_t(v_t) = 0, r\leq R.$ According to Lemma~\ref{lem:instantaneousError}, we have
    \begin{align*}
        & M_t Z_t \tau_t(r + |H_t(\theta_t)| - \frac{\tau_t}{2}\|\nabla H_t(\theta_t)\|^2_*)  + L_t Z_t \tau_t(r - |H_t(\theta_t)| - \frac{\tau_t}{2}\|\nabla H_t(\theta_t)\|^2_*) \\
        & \leq D_{\mathcal{R}}(v_t,\theta_t)-D_{\mathcal{R}}(v_t,\theta_{t+1}).
    \end{align*}
    By taking $\tau_t = f^r_t (\theta_t) / \| \nabla H_t(\theta_t)\|^2_*,$ we can obtain
    \begin{align*}
        & M_t Z_t \tau_t(r + |H_t(\theta_t)| - \frac{\tau_t}{2}\|\nabla H_t(\theta_t)\|^2_*)  + L_t Z_t \tau_t(r - |H_t(\theta_t)| - \frac{\tau_t}{2}\|\nabla H_t(\theta_t)\|^2_*)\\
        & = M_t Z_t \tau_t(r + |H_t(\theta_t)| - \frac{1}{2}(r + |H_t(\theta_t)|)  + L_t Z_t \tau_t(r - |H_t(\theta_t)| - \frac{1}{2}(r - |H_t(\theta_t)|) \\
        & = \frac{1}{2} M_t Z_t \tau_t(r + |H_t(\theta_t)|)  + \frac{1}{2} L_t Z_t(r - |H_t(\theta_t)|) \\
        & \geq \frac{1}{2} M_t Z_t \tau_t(r + |H_t(\theta_t)|).
    \end{align*}
    The last inequality comes from the definition of $L_t.$ Therefore
    \begin{align*}
        & \frac{1}{2} M_t Z_t \tau_t(r+ |H_t(\theta_t)|) \leq D_{\mathcal{R}}(v_t,\theta_t)-D_{\mathcal{R}}(v_t,\theta_{t+1}). 
    \end{align*}
    From Assumption~4, we know that 
    \begin{equation*}
        M_t\tau_t = M_t\frac{f^r_t(\theta_t)}{\|\nabla H_t(\theta_t)\|^2_*}\geq M_t\frac{f^r_t(\theta_t)}{G^2}\geq M_t\frac{r}{G^2},
    \end{equation*}
    we thus have
    \begin{align*}
       \frac{r}{2G^2}M_t Z_t (r + |H_t(\theta_t)| ) \leq D_{\mathcal{R}} (v_t,\theta_t) - D_{\mathcal{R}} (v_t,\theta_{t+1}).
    \end{align*}
\end{proof}
With the above Lemmas, we are now ready to proof the Lemma~1.
\begin{proof}[Proof of Lemma~1]

    First, by the condition $D_{\mathcal{R}} (x,z)-D_{\mathcal{R}}(y,z) \leq \gamma \|x-y\|, \forall x,y,z \in \mathcal{K},$ we have
    \begin{align*}
        \sum_{t=1}^T D_{\mathcal{R}} (v_t,\theta_t) - D_{\mathcal{R}} (v_t,\theta_{t+1}) & \leq D_{\mathcal{R}} (v_1,\theta_1) + \sum_{t=1}^{T-1} \left( D_{\mathcal{R}} (v_{t+1},\theta_{t+1}) - D_{\mathcal{R}} (v_t,\theta_{t+1})\right)\\
        & = \epsilon_v + \gamma  \sum_{t=1}^{T-1} \|v_{t+1}-v_t\| \\
        & = \epsilon_v +  \gamma C_T
    \end{align*}
    Then, by the definition of OSAMD algorithm and Lemma~\ref{lem:PA1InstantaneousError}, we have
    \begin{align*}
        \mathbb{E}[\sum_{t=1}^T M_t] & = \frac{1}{r} \mathbb{E}[\sum_{t=1}^T M_t Z_t (r + |H_t(\theta_t)| ) ] \\
        & = \frac{2G^2}{r^2} \mathbb{E}[\sum_{t=1}^T \frac{r}{2G^2}M_t Z_t (r + |H_t(\theta_t)| )]\\ 
        & \leq  \frac{2G^2}{r^2} \mathbb{E}[\sum_{t=1}^T D_{\mathcal{R}} (v_t,\theta_t) - D_{\mathcal{R}} (v_t,\theta_{t+1})]\\ 
        & \leq \frac{2G^2}{r^2} (\epsilon_v + \gamma C_T)\\
        & = \frac{2G^2}{\sigma^2} (\epsilon_v + \gamma C_T),
    \end{align*}
    where $r=\sigma$. We thus end the proof.
\end{proof}

\subsubsection{General Case}
Here, we provide the analysis for the pseudolabel errors of the general case, where we do not assume that the data distribution $P_t$ is $100\%$ separated within a margin. We first present the following lemma.
\begin{lemma}\label{lem:PA2InstantaneousError}
    Take the same assumptions as Lemma~1. Then for Algorithm~1, let $\tau_t = \min\{C,\frac{f^r_t(\theta_t)}{\| \nabla H_t(\theta_t)\|^2_*}\} $, the following inequality holds
    \begin{align*}
       & \min \{C,\frac{r}{G^2}\}\frac{1}{2}M_t Z_t (r + |H_t(\theta_t)| ) \leq D_{\mathcal{R}} (v_t,\theta_t) - D_{\mathcal{R}} (v_t,\theta_{t+1}) +  C f_t^r(v_t) ,
    \end{align*}
    for $r\leq R.$
\end{lemma}
\begin{proof}
    First, according to Lemma~\ref{lem:instantaneousError}, we have
    \begin{align*}
        & M_t Z_t \tau_t(r + |H_t(\theta_t)| - \frac{\tau_t}{2}\|\nabla H_t(\theta_t)\|^2_* ) + L_t Z_t \tau_t(r - |H_t(\theta_t)| - \frac{\tau_t}{2}\|\nabla H_t(\theta_t)\|^2_*) \\
        &  \leq D_{\mathcal{R}}(v_t,\theta_t)-D_{\mathcal{R}}(v_t,\theta_{t+1}) + \tau_t f^r_t(v_t).
    \end{align*}
    Since we take 
    \begin{equation*}
        \tau_t =  \min\{C,\frac{f^r_t (\theta_t)}{\|\nabla H_t(\theta_t)\|^2_*}\}\leq  f^r_t (\theta_t)/ \|\nabla H_t(\theta_t)\|^2_*.
    \end{equation*}
    Similar to Lemma~\ref{lem:PA1InstantaneousError}, we have
    \begin{align*}
       & \frac{\tau_t}{2}M_t Z_t (r + |H_t(\theta_t)| ) \leq D_{\mathcal{R}} (v_t,\theta_t) - D_{\mathcal{R}} (v_t,\theta_{t+1}) +  \tau_t f_t^r(v_t).
    \end{align*}
    Since we know that
    \begin{align*}
        M_t\tau_t =  M_t\min\{C,\frac{f^r_t (\theta_t)}{\|\nabla H_t(\theta_t)\|^2_*}\}\leq  M_t\min \{{C},\frac{r}{G^2}\}.
    \end{align*}
    Further, by $\tau_t\leq C.$ We therefore have
    \begin{align*}
       & \min \{C,\frac{r}{G^2}\}\frac{1}{2}M_t Z_t (r + |H_t(\theta_t)| ) \leq D_{\mathcal{R}} (v_t,\theta_t) - D_{\mathcal{R}} (v_t,\theta_{t+1}) +  C f_t^r(v_t) .
    \end{align*}
\end{proof}

We are now ready to prove the Lemma~3.
\begin{proof}[Proof of Lemma~3]
    First by the proof of Lemma~1, we have
    \begin{align*}
        & \sum_{t=1}^T D_{\mathcal{R}} (v_t,\theta_t) - D_{\mathcal{R}} (v_t,\theta_{t+1})  \leq  \epsilon_v +  \gamma C_T.
    \end{align*}
    Then 
    \begin{align*}
        & \mathbb{E}[\sum_{t=1}^T M_t]  = \frac{1}{r} \mathbb{E}[\sum_{t=1}^T M_t Z_t (r + |H_t(\theta_t)| ) ] \\
        & = \frac{2}{r^2} \max\{\frac{r}{C},{G^2}\} \mathbb{E}[\sum_{t=1}^T \min \{C,\frac{r}{G^2}\}\frac{1}{2}M_t Z_t (r + |H_t(\theta_t)| )]\\ 
        & \leq \frac{2}{r^2} \max\{\frac{r}{C},{G^2}\} \mathbb{E}[\sum_{t=1}^T D_{\mathcal{R}} (v_t,\theta_t) - D_{\mathcal{R}} (v_t,\theta_{t+1})+ \sum_{t=1}^T C f_t^r(v_t)]\\ 
        & \leq \frac{2}{r^2} \max\{\frac{r}{C},{G^2}\}  (\epsilon_v + \gamma C_T + C \sum_{t=1}^T l_t^r(v_t)) \\
        & \leq \frac{2}{r^2} \max\{\frac{r}{C},{G^2}\}  (\epsilon_v + \gamma C_T + C T \alpha^*) \\
        & = \frac{2G^2}{\sigma^2} (\epsilon_v + \gamma C_T + \frac{\sigma}{G^2} T \alpha^*) ,
    \end{align*}
    where $r = \sigma$ and $C=\sigma/G^2$. The second inequality comes from $l_t^r(v_t) = \mathbb{E}[f_t^r(v_t)]$, and the last inequality comes from $l_t^r(v_t) \leq l_t^R(v_t) \leq \alpha^*.$ We thus end the proof.
\end{proof}

\subsection{Dynamic Regret Bound}
In this subsection, we begin to bound the dynamic regret. We will first provide necessary lemmas, and then use these lemmas to give the final proof.

We here give similar result as Lemma~\ref{lem:BeckandTeboulle} for property of the implicit gradient mirror descent.
\begin{lemma}\label{lem:ImplicitToBregman}
    Let $\mathcal{K}$ be a convex set in a Banach space $\mathcal{B},$ and regularizer $\mathcal{R}: \mathcal{K} \mapsto \mathbb{R}$ be a convex function, and let $D_{\mathcal{R}}(\cdot, \cdot)$ be the Bregman divergence induced by $\mathcal{R}$. Then, any update of the form for convex function $f$
    \begin{equation*}
        w^{*}=\underset{w \in \mathcal{K}}{\arg \min }\left\{f(w)+D_{\mathcal{R}}(w, c)\right\}
    \end{equation*}
    satisfies the following inequality
    \begin{equation*}
        \left\langle w^{*}-d, \nabla f(w^*)\right\rangle \leq D_{\mathcal{R}}(d, c)-D_{\mathcal{R}}\left(d, w^{*}\right)-D_{\mathcal{R}}\left(w^{*}, c\right)
    \end{equation*}
    for any $d \in \mathcal{K}.$
\end{lemma}
\begin{proof}
    By the convexity of $f$ and $\mathcal{R},$ it is easy to verify the convexity of $D_\mathcal{R}.$ Then $f(w)+D_{\mathcal{R}}(w, c)$ is convex, and by the optimality of $w^*$ and KKT condition (Theorem 2.2~\citep{hazan2016introduction}), we have
  \begin{equation*}
      \langle d - w^*, \nabla_{w^*}(f(w^*) + D_{\mathcal{R}} (w^*,c) )\rangle  \geq 0, \forall d \in \mathcal{K}.
  \end{equation*}
  By the definition of Bregman divergence, we can see that
  \begin{equation*}
      \langle d - w^*, \nabla f(w^*) +  \nabla R(w^*) - \nabla R(c) \rangle \geq 0, \forall d \in \mathcal{K}.
  \end{equation*}
  Thus we obtain
  \begin{equation*}
      \langle w^* - d, \nabla f(w^*)  \rangle \leq  \langle d - w^* , \nabla R(w^*) - \nabla R(c)  \rangle, \forall d \in \mathcal{K}. 
  \end{equation*}
  The rest is the same with Lemma~\ref{lem:BeckandTeboulle}. For completeness, we present the proof here. By the definition of Bergman divergence, we know that
  \begin{align*}
    & D_{\mathcal{R}}(d, c)-D_{\mathcal{R}}\left(d, w^{*}\right)-D_{\mathcal{R}}\left(w^{*}, c\right) \\
    & = R(d) - R(c) - \nabla R(c)(d-c)  - (R(d) - R(w^*) - \nabla R(w^*)(d-w^*)) \\&\quad - (R(w^*) - R(c) - \nabla R(c)(w^*-c))\\
    & = - \langle \nabla R(c), d \rangle + \langle \nabla R(w^*), d - w^*  \rangle  + \langle \nabla R(c), w^* \rangle\\
    & = \langle d - w^* , \nabla R(w^*) - \nabla R(c)  \rangle. 
  \end{align*}
  Therefore, we finally conclude for $\forall d\in \mathcal{K}$
  \begin{equation*}
        \left\langle w^{*}-d, \nabla f(w^*)\right\rangle \leq D_{\mathcal{R}}(d, c)-D_{\mathcal{R}}\left(d, w^{*}\right)-D_{\mathcal{R}}\left(w^{*}, c\right).
    \end{equation*}
\end{proof}

Usually, previous works handle the noisy gradient by the property of $\mathbb{E}[f_t(w_t)|w_t]=l_t(w_t)$. However, since $x_t$ and $w_t$ are mutually depended, we have $\mathbb{E}[f_t(w_t)|w_t]\neq l_t(w_t).$ Fortunately, using the linearity of expectation and by the law of total expectation, we wouldn't need to handle the conditional expectation directly, as the following Lemma.

\begin{lemma}[Restatement of Lemma~2]
  For algorithm~1. We have for $t= 1, \dots, T $
  \begin{equation}\label{eq:noiseGradient}
    \mathbb{E} [ l_t (w_t) - l_t(u_t)] \leq \mathbb{E} [\langle \nabla f_t(w_t), w_t - u_t \rangle] + \mathbb{E}[2(LD+G)\|w_t - \hat{w}_t\|].
  \end{equation}
\end{lemma}
\begin{proof}
  First, from the above condition and by the convexity, we have
  \begin{align*}
              \mathbb{E}[l_t (w_t) - l_t (u_t) ]
              & \leq \mathbb{E}[\langle \nabla l_t (w_t), w_t - u_t \rangle]\\
              & = \mathbb{E}[\langle  \mathbb{E}[ \nabla f_t(w_t) | w_t], w_t - u_t \rangle] + \mathbb{E}[\langle  \nabla l_t (w_t) - \mathbb{E}[ \nabla  f_t(w_t) | w_t], w_t - u_t \rangle]\\
              & = \mathbb{E}[\mathbb{E}[\langle \nabla f_t(w_t) , w_t - u_t \rangle| w_t]] + \mathbb{E}[\langle  \nabla l_t (w_t) - \mathbb{E}[ \nabla  f_t(w_t) | w_t], w_t - u_t \rangle]\\
              & = \mathbb{E}[\langle \nabla f_t(w_t) , w_t - u_t \rangle] + \mathbb{E}[\langle  \nabla l_t (w_t) - \mathbb{E}[ \nabla  f_t(w_t) | w_t], w_t - u_t \rangle].
  \end{align*}
  Since $\hat{w}_t$ does not depend on $x_t$, we know that $ \nabla l_t (\hat{w}_t) = \mathbb{E}[ \nabla  f_t(\hat{w}_t) | \hat{w}_t].$ We then begin to estimate the second term, which can be decomposed as
  \begin{align*}
              \mathbb{E}[\langle  \nabla l_t (w_t) - \mathbb{E}[ \nabla  f_t(w_t) | w_t], w_t - u_t \rangle] & = \underbrace{\mathbb{E}[ \langle \nabla l_t(w_t) - \nabla l_t (\hat{w}_t), w_t - u_t \rangle ]}_{\text {term } \mathrm{A}} \\
              & \quad + \underbrace{\mathbb{E}[ \langle \mathbb{E}[\nabla f_t (\hat{w}_t)|\hat{w}_t]- \nabla f_t (\hat{w}_t), w_t - u_t \rangle ]}_{\text {term } \mathrm{B}} \\
              & \quad + \underbrace{\mathbb{E}[ \langle \nabla f_t (\hat{w}_t) - \mathbb{E}[\nabla f_t(w_t)|w_t], w_t - u_t \rangle ]}_{\text {term } \mathrm{C}}.
  \end{align*}
  We next bound each term step by step. First, by the assumption of $L$-smoothness (Assumption~4), we can bound the term A by
    \begin{align*}
    \text {term } \mathrm{A} & = \mathbb{E}[\langle  \nabla l_t (w_t) - \nabla l_t (\hat{w}_t), w_t - u_t \rangle]\\
                                          & \leq \mathbb{E}[\|\nabla l_t (w_t) - \nabla l_t (\hat{w}_t)\|\|w_t - u_t\|]\\
                                          & \leq \mathbb{E}[L\|w_t-\hat{w}_t\|\|w_t - u_t\|]\\
                                          & \leq \mathbb{E}[LD\|w_t-\hat{w}_t\|].
  \end{align*}
  The first inequality comes from H\"older inequality, and the last one is from the bounded space (Assumption~5). Second, by the law of total expectation, we have $\mathbb{E}[ \langle\mathbb{E}[\nabla f_t (\hat{w}_t)|\hat{w}_t], \hat{w}_t - u_t \rangle ] = \mathbb{E}[\langle\nabla f_t (\hat{w}_t), \hat{w}_t - u_t \rangle ]$, hence the term B can be bounded as
  \begin{align*}
  \text {term } \mathrm{B} & = \mathbb{E}[ \langle\mathbb{E}[\nabla f_t (\hat{w}_t)|\hat{w}_t] - \nabla f_t (\hat{w}_t), w_t - u_t \rangle ] \\
  & = \mathbb{E}[ \langle\mathbb{E}[\nabla f_t (\hat{w}_t)|\hat{w}_t] - \nabla f_t (\hat{w}_t), \hat{w}_t - u_t \rangle ] + \mathbb{E}[ \langle\nabla l_t (\hat{w}_t) - \nabla f_t (\hat{w}_t), w_t - \hat{w}_t \rangle ] \\
  & = \mathbb{E}[ \langle\mathbb{E}[\nabla f_t (\hat{w}_t)|\hat{w}_t], \hat{w}_t - u_t \rangle ] - \mathbb{E}[ \langle \nabla f_t (\hat{w}_t), \hat{w}_t - u_t \rangle ] \\
  & \quad + \mathbb{E}[ \langle\nabla l_t (\hat{w}_t) - \nabla f_t (\hat{w}_t), w_t - \hat{w}_t \rangle ] \\
  & \leq 0 + \mathbb{E}[\|\nabla l_t (\hat{w}_t) - \nabla f_t (\hat{w}_t)\|\|w_t - \hat{w}_t\|]\\
  & \leq 0 + \mathbb{E}[2G\|w_t - \hat{w}_t\|].
  \end{align*}
  The last inequality is from H\"older inequality and bounded gradient (Assumption~4). Finally, we have for term C
  \begin{align*}
  \text {term } \mathrm{C} & = \mathbb{E}[ \langle \nabla f_t (\hat{w}_t) - \mathbb{E}[\nabla f_t(w_t)|w_t], w_t - u_t \rangle ]\\
  &=\mathbb{E}[ \langle\nabla f_t (\hat{w}_t), w_t - u_t \rangle ] - \mathbb{E}[ \mathbb{E}[\nabla f_t(w_t)|w_t], w_t - u_t \rangle ] \\
  & = \mathbb{E}[ \langle\nabla f_t (\hat{w}_t), w_t - u_t \rangle ] - \mathbb{E}[ \mathbb{E}[\nabla f_t(w_t), w_t - u_t \rangle|w_t] ] \\
  & = \mathbb{E}[ \langle\nabla f_t (\hat{w}_t) - \nabla f_t (w_t) , w_t - u_t \rangle ] \\
  & \leq \mathbb{E}[ L\|w_t - u_t\| \|w_t - \hat{w}_t\| ]\\
  & \leq \mathbb{E}[ LD \|w_t - \hat{w}_t\| ].
  \end{align*}
  The last two inequality is from H\"older inequality and $L$-smoothness, and the last one is from bounded space (Assumption~5). By summing up, we get
  \begin{equation}
    \mathbb{E} [ l_t (w_t) - l_t(u_t)] \leq \mathbb{E} [\langle \nabla f_t(w_t), w_t - u_t \rangle] + \mathbb{E}[2(LD+G)\|w_t - \hat{w}_t\|].
  \end{equation}
\end{proof}

The continual domain shift (Assumption~1) and bounded function (Assumption~5) lead to the temporal variability condition in the online learning (as shown in Proposition~\ref{prop:TotalVariation}). It is not easy to analyze the dynamic regret (temporal variability form) directly, thus we first provide the path-length version as the following lemma.
\begin{lemma}\label{lem:pathlenghtbound}
    Under the same assumption as Lemma~1. If we choose $\eta \leq \frac{1}{4(LD+G)},$ Algorithm~1 has the following bound
  \begin{align*}
    & \mathbb{E} [\sum_{t=1}^{T}l_t(w_t)] - \sum_{t=1}^{T}l_t(u_t) \leq (2\eta G^2 + 2G D)  \mathbb{E} [ \sum_{t=1}^T  M_t  ] + \sum_{t=1}^{T-1} \frac{1}{\eta} \gamma \|u_{t+1}-u_t\| + \frac{1}{\eta}D_{\mathcal{R}}(u_{1}, \hat{w}_{1}),\\
  \end{align*}
  for all $u_1,\dots,u_T \in \mathcal{K}.$ 
\end{lemma}

\begin{proof}
    
    Denote $\hat{f}_t(\cdot) = f(\cdot;x_t,\hat{y}_t),\tilde{f}_t(\cdot) = f(\cdot;x_t,\tilde{y}_t),{f}_t(\cdot) = f(\cdot;x_t,{y}_t)$ for simplicity. By Lemma~2, we have
  \begin{align*}
    & \mathbb{E} [\sum_{t=1}^{T}l_t(w_t)] - \sum_{t=1}^{T}l_t(u_t) = \mathbb{E} [\sum_{t=1}^{T}l_t(w_t)- \sum_{t=1}^{T}l_t(u_t)]\\
    & \leq \mathbb{E} [\langle \nabla f_t(w_t), w_t - u_t \rangle  + 2(LD+G)\|w_t - \hat{w}_t\|]\\
    & = \mathbb{E} [ \underbrace{\langle \nabla f_t(w_t) - \nabla \hat{f}_t(w_t), w_t - \hat{w}_{t+1}\rangle}_{\text {term } \mathrm{A}}  + \underbrace{\langle \nabla \hat{f}_t(w_t) , w_{t} - \hat{w}_{t+1} \rangle}_{\text {term } \mathrm{B}}  + \underbrace{\langle \nabla \tilde{f}_t(w_t), \hat{w}_{t+1} - u_t \rangle}_{\text {term } \mathrm{C}} \\
    & \quad + \underbrace{\langle \nabla {f}_t(w_t) - \nabla \tilde{f}_t(w_t), \hat{w}_{t+1} - u_t \rangle}_{\text {term } \mathrm{D}} + 2(LD+G)\|w_t - \hat{w}_t\|].
  \end{align*}
  We next bound each term step by step. First, we can bound term A in terms of the pseudolabel errors.
  \begin{align*}
      \text {term } \mathrm{A} &  = \langle \nabla f_t(w_t) - \nabla \hat{f}_t(w_t), w_t - \hat{w}_{t+1}\rangle \\
      & = M_t\langle \nabla f_t(w_t) - \nabla \hat{f}_t(w_t), w_t - \hat{w}_{t+1}\rangle\\
      & \leq M_t \|\nabla f_t(w_t) - \nabla \hat{f}_t(w_t)\|_* \|w_t - \hat{w}_{t+1}\|\\
      & \leq M_t 2G \|w_t - \hat{w}_{t+1}\| \\
      & \leq 2\eta M_t G^2 + \frac{1}{2\eta}\|w_t - \hat{w}_{t+1}\|^2.
  \end{align*}
  The first inequality holds due to H\"older inequality, and the last one holds due to the fact that $2ab\leq {2\eta} a^2 + \frac{1}{2\eta}b^2$ for $\eta >0$ and $M_t^2 =M_t.$ By Lemma~\ref{lem:ImplicitToBregman}, we could bound term B
  \begin{align*}
      \text {term } \mathrm{B} & = \langle \nabla \hat{f}_t(w_t) , w_{t} - \hat{w}_{t+1} \rangle \\
      & \leq \frac{1}{\eta} (D_{\mathcal{R}}(\hat{w}_{t+1}, \hat{w}_{t})-D_{\mathcal{R}}\left(\hat{w}_{t+1}, w_{t}\right)-D_{\mathcal{R}}\left(w_{t}, \hat{w}_{t}\right))\\
      & \leq \frac{1}{\eta} (D_{\mathcal{R}}(\hat{w}_{t+1}, \hat{w}_{t})-\frac{1}{2}\|w_t - \hat{w}_{t+1}\|^2-\frac{1}{2}\|w_t - \hat{w}_{t}\|^2).
  \end{align*}
  The last inequality is from the strongly convexity of regularizer $R.$ By Lemma~\ref{lem:BeckandTeboulle}, we next bound term C
  \begin{align*}
      & \text {term } \mathrm{C} = \langle \nabla \tilde{f}_t(w_t), \hat{w}_{t+1} - u_t \rangle \leq \frac{1}{\eta} (D_{\mathcal{R}}(u_t, \hat{w}_{t})-D_{\mathcal{R}}\left(u_t, \hat{w}_{t+1}\right)-D_{\mathcal{R}}\left(\hat{w}_{t+1}, \hat{w}_{t}\right) ) .
  \end{align*}
  From the Algorithm~1, we know that only when the pseudolabel makes mistake and the active agent does not query the label, $\tilde{y}\neq y.$ Similar to term A, we have the bound for term D
     \begin{align*}
      & \text {term } \mathrm{D} = \langle \nabla {f}_t(w_t) - \nabla \tilde{f}_t(w_t), \hat{w}_{t+1} - u_t \rangle \\
      & =  \langle  M_t(1-Z_t) (\nabla f_t(w_t) -  \nabla \hat{f}_t(w_t) ), \hat{w}_{t+1} - u_t \rangle\\
      & \leq  M_t(1-Z_t) \|\nabla f_t(w_t) -  \nabla \hat{f}_t(w_t)\|_* \|\hat{w}_{t+1} - u_t \| \\
      & \leq 2 M_t(1-Z_t) G D\\
      & \leq 2 M_t GD.
  \end{align*}
  
  The first inequality is from the H\"older inequality, and second inequality is from the Assumption~4 and Assumption~5. Also from the H\"older inequality, the last term can be bounded as $2(LD+G)\|w_t-\hat{w}_t\|\leq 2(LD+G)^2\eta + \frac{1}{2\eta}\|w_t-\hat{w}_t\|^2$. Finally, we have the path-length version of dynamic regret bound
  \begin{align*}
    & \mathbb{E} [\sum_{t=1}^{T}l_t(w_t)] - \sum_{t=1}^{T}l_t(u_t)  \\
    & \leq  \mathbb{E} [ \sum_{t=1}^T \text {term } \mathrm{A} + \text {term } \mathrm{B} + \text {term } \mathrm{C} + \text {term } \mathrm{D} + 2(LD+G)^2\eta + \frac{1}{2\eta}\|w_t-\hat{w}_t\|^2]\\
    & \leq \mathbb{E} [ \sum_{t=1}^T 2\eta M_t G^2 + \sum_{t=1}^T 2 M_t G D + \sum_{t=1}^T  \frac{1}{\eta} (D_{\mathcal{R}}(u_t, \hat{w}_{t})-D_{\mathcal{R}}\left(u_t, \hat{w}_{t+1}\right))] + 2(LD+G)^2\eta T\\
    & = (2\eta G^2 + 2G D)  \mathbb{E} [ \sum_{t=1}^T  M_t  ] + \mathbb{E} [ \sum_{t=1}^T  \frac{1}{\eta} (D_{\mathcal{R}}(u_t, \hat{w}_{t})-D_{\mathcal{R}}\left(u_t, \hat{w}_{t+1}\right))] + 2(LD+G)^2\eta T.
  \end{align*}
  By the condition $D_R (x,z)-D_R(y,z) \leq \gamma \|x-y\|, \forall x,y,z \in \mathcal{K},$ we can get
  \begin{align*}
    &\mathbb{E} [ \sum_{t=1}^T  \frac{1}{\eta} (D_{\mathcal{R}}(u_t, \hat{w}_{t})-D_{\mathcal{R}}\left(u_t, \hat{w}_{t+1}\right))] \\
    & \leq \mathbb{E} [ \sum_{t=1}^{T-1}  \frac{1}{\eta} (D_{\mathcal{R}}(u_{t+1}, \hat{w}_{t+1})-D_{\mathcal{R}}\left(u_t, \hat{w}_{t+1}\right))] + \frac{1}{\eta}D_{\mathcal{R}}(u_{1}, \hat{w}_{1}) \\
    & \leq \sum_{t=1}^{T-1} \frac{1}{\eta} \gamma \|u_{t+1}-u_t\| + \frac{1}{\eta}D_{\mathcal{R}}(u_{1}, \hat{w}_{1}).
  \end{align*}
  From the above, we thus have
  \begin{align*}
    & \mathbb{E} [\sum_{t=1}^{T}l_t(w_t)] - \sum_{t=1}^{T}l_t(u_t)  \leq (2\eta G^2 + 2G D)  \mathbb{E} [ \sum_{t=1}^T  M_t  ]+ 2(LD+G)^2\eta T + \frac{\epsilon_w}{\eta} + \sum_{t=1}^{T-1} \frac{1}{\eta} \gamma \|u_{t+1}-u_t\| .
  \end{align*}
\end{proof}

Next, we give a general version of our regret bound analysis, concluding both the separable case and the general case.
\begin{theorem}\label{thm:generalBound}
    Take the same assumptions as Lemma~1, Algorithm~1 has the following bound for $\eta \leq \frac{1}{4(LD+G)}$
    \begin{align*}
      &\text{\rm D-Regret}(\{P_t\},T) \leq (2\eta G^2 + 2G D)  \mathbb{E} [ \sum_{t=1}^T  M_t  ] + 2(LD+G)^2\eta + \frac{\epsilon_w + \gamma D}{\eta} T + 4\sqrt{\frac{\gamma D T F V_T }{\eta}}.
    \end{align*}
\end{theorem}
\begin{proof}[Proof Sketch]
    This proof shares the same technique with~\citet{zhang2020simple}, which introduces detailed proof for converting path-length bound to temporal variability bound. The key of this converting is to specify a sequence of $\left\{u_{1}, \ldots, u_{T}\right\}$ in the following way.
    \begin{align*}
        \left\{u_{1}, \ldots, u_{T}\right\} = \left\{w_{1}^*,\underbrace{ w_{\mathcal{I}_{2}}^{\star}, \ldots, w_{\mathcal{I}_{2}}^{\star}}_{\Delta \text{ times}}, \underbrace{w_{\mathcal{I}_{3}}^{\star},\ldots, w_{\mathcal{I}_{3}}^{\star}}_{\Delta \text{ times}},\ldots,\underbrace{w_{\mathcal{I}_{\lceil T-1 / \Delta \rceil + 1}}^{\star}, \ldots,w_{\mathcal{I}_{\lceil T-1 / \Delta \rceil + 1}}^{\star}}_{\Delta \text{ times}}\right\}.
    \end{align*}
    This piece-wise stationary sequence starts with $w_1^*$ and next changes every $\Delta \in[T]$ iterations. We specify $u_{t}$ as the best fixed decision $w_{\mathcal{I}_{i}}^{\star}=\arg \min _{w \in \mathcal{K}} \sum_{t \in \mathcal{I}_{i}} l_{t}(w)$ of the corresponding interval $\mathcal{I}_{i} .$ The rest is same as the proof of Lemma 2 in \citep{zhang2020simple}.
\end{proof}

Within Theorem~\ref{thm:generalBound}, it is simple to bound both the separable and the general (non-separable) cases.
\begin{proof}[Proof of Theorem~1]
    By the result of Lemma~1 we know that
    \begin{align*}
        \mathbb{E}[\sum_{t=1}^T M_t] \leq  \frac{2G^2}{\sigma^2} (\gamma C_T + \epsilon_v).
    \end{align*}
    Plugging in Theorem~\ref{thm:generalBound}, we then have
    \begin{align*}
      \text{\rm D-Regret}(\{P_t\},T) & \leq (2\eta G^2 + 2G D)  \mathbb{E} [ \sum_{t=1}^T  M_t  ]+ 2(LD+G)^2\eta T + \frac{\epsilon_w + \gamma D}{\eta} + 4\sqrt{\frac{\gamma D T F V_T }{\eta}}\\
        & \leq \frac{4(\eta G^4 + G^3 D)}{\sigma^2} (\gamma C_T + \epsilon_v) + 2(LD+G)^2\eta T + \frac{\epsilon_w + \gamma D}{\eta} + 4\sqrt{\frac{\gamma D T F V_T }{\eta}}.
    \end{align*}
\end{proof}
Similarly, we can generate it to the separable case.
\begin{proof}[Proof of Theorem~4]
    By the result of Lemma~3, we know that
    \begin{align*}
        &\mathbb{E}[\sum_{t=1}^T M_t]  \leq \frac{2G^2}{\sigma^2} ( \gamma C_T +\epsilon_v + \frac{\sigma}{G^2} T \alpha^*).
    \end{align*}
    Plugging in Theorem~\ref{thm:generalBound}, we then have
    \begin{align*}
      & \text{\rm D-Regret}(\{P_t\},T) \\
      &\leq (2\eta G^2 + 2G D)  \mathbb{E} [ \sum_{t=1}^T  M_t  ] + 2(LD+G)^2\eta T + \frac{\epsilon_w + \gamma D}{\eta} + 4\sqrt{\frac{\gamma D T F V_T }{\eta}}\\
      & \leq \frac{4(\eta G^4 + G^3 D)}{\sigma^2} (\gamma C_T + \epsilon_v + \frac{\sigma}{G^2} T \alpha^* ) + 2(LD+G)^2\eta T + \frac{\epsilon_w + \gamma D}{\eta} + 4\sqrt{\frac{\gamma D T F V_T }{\eta}}.
    \end{align*}
\end{proof}

\subsection{Lower bound}
Here, we show the lower bound for Theorem~2.
\begin{proof}[Proof of Theorem~2] We here create an example for the worst case that satisfies our assumptions.
    
    Assume we have two data points $(-1,0)$ and $(1,0)$ with the same probability $1/2$ to be sampled. Then we let $(-1,0)$ to be class $1$ and $(1,0)$ to be class $-1$ when $t = 1,\dots,\frac{T}{2},$ and let $(-1,0)$ to be class $-1$ and $(1,0)$ to be class $1$ when $t = \frac{T}{2} + 1, \dots, T.$ We use the hinge loss $l_t = \max\{0,1-y_t w_t x_t\},$ and the decision space is $\{w | \|w\|_2 \leq 1\}.$ It is easy to verify that this setting satisfies our assumptions where $V_T\leq 2, C_T\leq 2.$
    
    For any unsupervised self-training algorithm that begins with a good initial $w = (-1,0)$, it is impossible to get the information for the label change in hindsight. Then the learner takes no update when $t = \frac{T}{2} + 1, \dots, T,$ and therefore suffers from $T/2$ regret, which is an order of $T.$

\end{proof}

\section{Extension to Multiclass}\label{supp:multiclass}
In this section, we extend the results to the multiclass case.
\subsection{Multiclass setting}
Denote $\mathcal{Y}$ to be the class set. The multiclass setting is slightly different from the binary setting, and we present the important formulations and assumptions of multiclass case as follows.

Denote the soft prediction over instance $x$ as $H(\theta;x),$ which outputs $|\mathcal{Y}|$ prediction scores:
\begin{equation*}
    H^{s}(\theta;x),s\in \mathcal{Y}.
\end{equation*} 
Denote $H(\theta;x_t)=H_t(\theta)$ for simplicity. In each round $t,$ the margin function is defined as 
\begin{equation*}
    \Psi_t(\theta) \defeq H_t^{y_t}(\theta) - \max_{s_t\neq y_t, s_t\in \mathcal{Y}}H_t^{s_t}(\theta),
\end{equation*}
which is the gap between the prediction score of the real class and the irrelevant class with the highest score. Assumption~2 is modified as 
\begin{assumption}[Multiclass Separation]\label{amp:separablemulticlass}
    There exists $\alpha^* > 0$ such that for each time step $t=1,\ldots,T$, the data distribution $P_t$ can be classified almost correctly with a margin $R$, i.e., there exists $v_t \in \mathcal{K}$ and a constant $\alpha^*$ such that
    \begin{equation*}
        \mathbb{E}_{(x_t,y_t)\sim P_t} [\max\{0,R - \Psi_t(v_t)\}] \leq \alpha^*.
    \end{equation*}
    We further assume that there exists a constant $C_T$ such that
    \begin{equation*}\label{asp:TotalVariationMulti}
        \sum_{t=1}^{T-1}\|v_t-v_{t+1}\| \leq C_T, 
    \end{equation*}
    i.e., the classifiers with margin $R$ change continually.
\end{assumption}
We further assume the convexity and bounded gradient on the margin function.
\begin{assumption}[Margin Function]\label{amp:convexitymulticlass}
    We assume that $-\Psi_t(\cdot)$ is convex, and $\|\nabla\Psi_t(\theta_t)\|_*\leq G.$
\end{assumption}
Others are the same as the binary class case.

\subsection{Multiclass OSAMD}
\begin{algorithm}[tb]
   \caption{Multiclass Online Self Adaptive Mirror Descent (MOSAMD)}
   \label{algorithm:MOAST}
\begin{algorithmic}
\footnotesize
   \STATE {\bfseries Input:} Active probability controller $\sigma,$ aggressive step size $\tau_t,$ conservative step size $\eta,$ initial data.
   \STATE {\bfseries Initial:} Learn from initial data, get aggressive model $\theta_1$ and conservative model $\hat{w}_1 $.
   \FOR {$t = 1,\dots, T$ }
   \STATE observe data sample $x_t$ 
   \STATE \textbf{pseudolabel:} 
   \STATE \quad give the pseudolabel provided by the aggressive model $\hat{y}_t = \arg\max_{s_t\in \mathcal{Y}}H_t^{s_t}(\theta_t)$
   \STATE \textbf{self-adaptation:} 
   \STATE \quad adapt the conservative model $w_t = \arg\min_{w\in \mathcal{K}} \eta f(w;x_t,\hat{y}_t) + D_{\mathcal{R}} (w,\hat{w}_{t})$ before making the decision
   \STATE \textbf{active query:} 
   \STATE \quad compute the confidence score $p_t = H_t^{\hat{y}_t}(\theta_t) - \max_{s_t\neq\hat{y}_t, s_t\in \mathcal{Y}}H_t^{s_t}(\theta_t)$
   \STATE \quad draw a Bernoulli random variable with probability $Z_t\sim Bernoulli(\sigma / (\sigma + p_t))$
   \IF {$Z_t = 1$}
   \STATE query label $y_t,$ compute the margin $\Psi_t(\theta_t) = H_t^{y_t}(\theta_t) - \max_{s_t\neq y_t, s_t\in \mathcal{Y}}H_t^{s_t}(\theta_t),$ and let $\tilde{y}_t = {y}_t$
   \STATE update the aggressive model $\theta_{t+1} = \arg\min_{\theta\in \mathcal{K}} - \tau_t \langle  \nabla \Psi_t(\theta_t), \theta\rangle + D_{\mathcal{R}} (\theta,\theta_t)$
   \ELSE 
   \STATE let $\theta_{t+1} = \theta_{t}$ and $\tilde{y}_t = \hat{y}_t$
   \ENDIF
   \STATE update the conservative model $\hat{w}_{t+1} = \arg\min_{w\in \mathcal{K}} \eta \langle \nabla f({w}_{t};x_t, \tilde{y}_t), w\rangle + D_{\mathcal{R}} (w,\hat{w}_{t})$ 
   \ENDFOR
\end{algorithmic}
\end{algorithm}
We here present the Multiclass OSAMD (MOSAMD) as Algorithm~\ref{algorithm:MOAST}. Specifically, we modify three areas in the binary OSAMD: 
\begin{enumerate}[\hspace{0em}1.]
    \item Pseudolabel is given by the class with the largest prediction scores $\hat{y}_t = \max_{s_t\in \mathcal{Y}}H_t^{s_t}(\theta_t)$;
    \item The uncertainty is measured by the difference between the largest and second largest prediction scores $p_t = H_t^{\hat{y}_t}(\theta_t) - \max_{s_t\neq\hat{y}_t, s_t\in \mathcal{Y}}H_t^{s_t}(\theta_t),$ which is designed to compute the query rate;
    \item The margin is defined by the gap between the prediction score of the real class and the irrelevant class with the highest score $\Psi_t(\theta) = H_t^{y_t}(\theta) - \max_{s_t\neq y_t, s_t\in \mathcal{Y}}H_t^{s_t}(\theta),$ based on which the pseudolabel (aggressive) model updates.
\end{enumerate}

\subsection{Analysis}
In this subsection, we analyze the theoretical performance of MOSAMD in the general case, which can be reduced to the separable case by setting $C =\infty$. We first begin with the pseudolabel errors bound, and then present the dynamic regret bound.
\subsubsection{Pseudolabel Errors}
Here, we present the theoretical bound of pseudolabel errors for the multiclass case.
\begin{lemma}[Pseudolabel Errors]\label{thm:labelErrorGeneralMulticlass}
  Let regularizer $\mathcal{R}: \mathcal{K} \mapsto \mathbb{R}$ be a 1-strongly convex function on $\mathcal{K}$ with respect to a norm $\|\cdot\|.$ Assume that $D_{\mathcal{R}}(\cdot, \cdot)$ satisfies $D_R (x,z)-D_R(y,z) \leq \gamma \|x-y\|, \forall x,y,z \in \mathcal{K}.$ Set 
  $
      \tau_t =  \min\{C,\frac{\max\{0,\sigma-\Psi_t(\theta_t)\}}{\| \nabla \Psi_t(\theta_t)\|^2_*}\},\sigma \leq R.
  $
  The expected number of pseudolabel errors made by Algorithm~\ref{algorithm:MOAST} is bounded by
  \begin{align*}
        & \mathbb{E}[\sum_{t=1}^T M_t]  \leq \frac{2G^2}{\sigma^2} (\gamma C_T + \epsilon_v + \frac{\sigma}{G^2} T \alpha^* ).
    \end{align*}
   where $M_t = \mathbf{1}_{\hat{y}_t \neq y_t}$ is the instantaneous mistake indicator.
\end{lemma}
The proof shares the same idea with the binary case. Before proving the theorem, we shall begin with important lemmas. 

Denote $f^r_t(\theta) = \max\{0,r-\Psi_t(\theta)\}.$ We first give the recursive of the multiclass case.
\begin{lemma}\label{lem:firstInequationMulticlass}
    Take the same assumptions as Lemma~\ref{thm:labelErrorGeneralMulticlass}. Then for algorithm~\ref{algorithm:MOAST}, the following inequality holds
    \begin{align*}
        &\tau_t r  - \tau_t \Psi_t(\theta_t) - \frac{\tau_t^2}{2}\|\nabla \Psi_t(\theta_t)\|^2_* \leq D_{\mathcal{R}} (v_t,\theta_t) - D_{\mathcal{R}} (v_t,\theta_{t+1}) + \tau_t f^r_t(v_t) .
    \end{align*}
    for $r >0.$
\end{lemma}

\begin{proof}
    First, by the definition of $f^r_t,$ we have
    \begin{align*}
        r -  f^r_t(v_t)  &  = r - \max\{0,r-\Psi_t(v_t)\} \leq  \Psi_t(v_t) \\
        & = \Psi_t(v_t) - \Psi_t(\theta_t)  + \Psi_t(\theta_t).  
    \end{align*}
    By the convexity of $-\Psi_t(\cdot),$ we have
    \begin{align*}
        \Psi_t(v_t)  - \Psi_t(\theta_t) & = - \Psi_t(\theta_t) - (-\Psi_t(v_t)) \\
        & \leq \langle -\nabla\Psi_t(\theta_t), \theta_t - v_t \rangle \\
        & \leq \langle - \nabla \Psi_t(\theta_t), \theta_{t+1} - v_t \rangle + \langle - \Psi_t(\theta_t), \theta_{t} - \theta_{t+1} \rangle.
    \end{align*}
    By the update rule of $\theta$ and Lemma~\ref{lem:BeckandTeboulle}, the first term can be bounded that
    \begin{align*}
        &\langle -\nabla\Psi_t(\theta_t), \theta_{t+1} - v_t \rangle \leq \frac{1}{\tau_t} (D_{\mathcal{R}}(v_t,\theta_t) - D_{\mathcal{R}}(v_t,\theta_{t+1}) - D_{\mathcal{R}} (\theta_{t+1},\theta_{t})  ).
    \end{align*}
    Due to H\"older inequality and the fact that $ab\leq \frac{\eta}{2} a^2 + \frac{1}{2\eta}G^2$ for $\eta >0,$ we obtain for the second term
    \begin{align*}
        \langle -\nabla\Psi_t(\theta_t), \theta_{t} - \theta_{t+1} \rangle &  \leq   \|\nabla\Psi_t(\theta_t)\|_* \|\theta_{t+1} - \theta_t\|\\
        & \leq \frac{\tau_t}{2} \| \nabla\Psi_t(\theta_t)\|_*^2 + \frac{1}{2\tau_t} \|\theta_{t+1} - \theta_t\|^2.
    \end{align*}
    Due to the strong convexity of regularizer $\mathcal{R},$ we have $D_{\mathcal{R}}(x, y) \geq \frac{1}{2}\|x-y\|^{2}$ for any $x, y \in \mathcal{X}$ \citep{mohri2018foundations}. Therefore, by plugging the above term, we obtain that
    \begin{align*}
        r -  f^r_t(v_t)  & \leq \frac{1}{\tau_t} (D_{\mathcal{R}}(v_t,\theta_t) - D_{\mathcal{R}}(v_t,\theta_{t+1}) -\frac{1}{2} \|\theta_{t+1} - \theta_t\|^2) \\
        & \quad + \frac{\tau_t}{2} \| \nabla\Psi_t(\theta_t)\|_*^2 + \frac{1}{2\tau_t} \|\theta_{t+1} - \theta_t\|^2 +  \Psi_t(\theta_t).
    \end{align*}
    By rearranging, we have
    \begin{align*}
        &\tau_t r  - \tau_t \Psi_t(\theta_t) - \frac{\tau_t^2}{2}\|\nabla\Psi_t(\theta_t)\|^2_* \leq D_{\mathcal{R}} (v_t,\theta_t) - D_{\mathcal{R}} (v_t,\theta_{t+1}) + \tau_t f^r_t(v_t) .
    \end{align*}
\end{proof}
Denote the instantaneous mistake by $M_t (w) = \mathbf{1}_{\hat{y}_t \neq y_t},$ and let $L_t (w) = \mathbf{1}_{\hat{y}_t = y_t, \Psi_t(w)\leq r}$ to be the indicator of the right decision but in the margin $r$, where $\mathbf{1}_{(\cdot)}$ is the indicator function. We then have the following relationship
\begin{lemma}\label{lem:instantaneousErrormulticlass}
    Take the same assumptions as Lemma~\ref{thm:labelErrorGeneralMulticlass}. For Algorithm~\ref{algorithm:MOAST}, let $\tau_t = 0$ if $ {f^r_t(\theta_t)}=0,$ then the following inequality holds for every $t$
    \begin{align*}
        & M_t Z_t \tau_t(r + |\Psi_t(\theta_t)| - \frac{\tau_t}{2}\|\nabla \Psi_t(\theta_t)\|^2_*)  + L_t Z_t \tau_t(r - |\Psi_t(\theta_t)| - \frac{\tau_t}{2}\|\nabla \Psi_t(\theta_t)\|^2_*) \\
        & \leq D_{\mathcal{R}}(v_t,\theta_t)-D_{\mathcal{R}}(v_t,\theta_{t+1}) + \tau_t f^r_t(v_t) 
    \end{align*}
    for $r>0.$
\end{lemma}
\begin{proof}
    From Lemma~\ref{lem:firstInequationMulticlass}, we know that
    \begin{align*}
        &\tau_t r  - \tau_t \Psi_t(\theta_t) - \frac{\tau_t^2}{2}\|\nabla\Psi_t(\theta_t)\|^2_* \leq D_{\mathcal{R}} (v_t,\theta_t) - D_{\mathcal{R}} (v_t,\theta_{t+1}) + \tau_t f^r_t(v_t) .
    \end{align*}
    Therefore, we can obtain
    \begin{align*}
        & M_t Z_t \tau_t(r + |\Psi_t(\theta_t) | - \frac{\tau_t}{2}\|\nabla \Psi_t(\theta_t) \|^2_*)  + L_t Z_t \tau_t(r - |\Psi_t(\theta_t) | - \frac{\tau_t}{2}\|\nabla \Psi_t(\theta_t) \|^2_*) \\
        & \leq M_t Z_t (D_{\mathcal{R}}(v_t,\theta_t)-D_{\mathcal{R}}(v_t,\theta_{t+1}) + \tau_t f^r_t(v_t)) + L_t Z_t (D_{\mathcal{R}}(v_t,\theta_t)-D_{\mathcal{R}}(v_t,\theta_{t+1}) + \tau_t f^r_t(v_t))\\
        & = (M_t + L_t) Z_t (D_{\mathcal{R}}(v_t,\theta_t)-D_{\mathcal{R}}(v_t,\theta_{t+1}) + \tau_t f^r_t(v_t)).
    \end{align*}
    From Algorithm~\ref{algorithm:MOAST}, we know that if $Z_t=0,$ then $\tau_t = 0, \theta_t = \theta_{t+1}.$ And if $M_t + L_t =0,$ we get $\Psi_t(\theta_t) \geq r,$ then $\tau_t=0, \theta_t = \theta_{t+1}.$ Therefore, we have
    \begin{align*}
        (M_t + L_t) Z_t (D_{\mathcal{R}}(v_t,\theta_t)-D_{\mathcal{R}}(v_t,\theta_{t+1}) + \tau_t f^r_t(v_t)) = D_{\mathcal{R}}(v_t,\theta_t)-D_{\mathcal{R}}(v_t,\theta_{t+1}) + \tau_t f^r_t(v_t).
    \end{align*}
    We finally get
    \begin{align*}
        & M_t Z_t \tau_t(r + |\Psi_t(\theta_t))| - \frac{\tau_t}{2}\|\nabla \Psi_t(\theta_t)\|^2_*)  + L_t Z_t \tau_t(r - |\Psi_t(\theta_t)| - \frac{\tau_t}{2}\|\nabla \Psi_t(\theta_t)\|^2_*) \\
        & \leq D_{\mathcal{R}}(v_t,\theta_t)-D_{\mathcal{R}}(v_t,\theta_{t+1}) + \tau_t f^r_t(v_t). 
    \end{align*}
\end{proof}
Next, we give a similar result as Lemma~\ref{lem:PA2InstantaneousError} of binary case.
\begin{lemma}\label{lem:PA2InstantaneousErrorMulticlass}
    Take the same assumptions as Lemma~\ref{thm:labelErrorGeneralMulticlass}. Then for Algorithm~\ref{algorithm:MOAST}, let $\tau_t = \min\{C,\frac{f^r_t(\theta_t)}{\| \nabla \Psi_t(\theta_t)\|^2_*}\}.$ then the following inequality holds
    \begin{align*}
       & \min \{C,\frac{r}{G^2}\}\frac{1}{2}M_t Z_t (r + p_t ) \leq D_{\mathcal{R}} (v_t,\theta_t) - D_{\mathcal{R}} (v_t,\theta_{t+1}) +  \tau_t f_t^r(v_t) ,
    \end{align*}
    for $r\leq R.$
\end{lemma}
\begin{proof}
    First, according to Lemma~\ref{lem:instantaneousErrormulticlass}, we have
    \begin{align*}
        & M_t Z_t \tau_t(r + |\Psi_t(\theta_t)| - \frac{\tau_t}{2}\|\nabla \Psi_t(\theta_t)\|^2_* )\quad + L_t Z_t \tau_t(r - |\Psi_t(\theta_t)| - \frac{\tau_t}{2}\|\nabla \Psi_t(\theta_t)\|^2_*) \\
        &  \leq D_{\mathcal{R}}(v_t,\theta_t)-D_{\mathcal{R}}(v_t,\theta_{t+1}) + \tau_t f^r_t(v_t).
    \end{align*}
    Since we take 
    \begin{equation*}
        \tau_t =  \min\{C,\frac{f^r_t (\theta_t)}{\|\nabla \Psi_t(\theta_t)\|^2_*}\}\leq  f^r_t (\theta_t)/ \|\nabla \Psi_t(\theta_t)\|^2_*.
    \end{equation*}
    Similar to Lemma~\ref{lem:PA1InstantaneousError}, we have
    \begin{align*}
       & \frac{\tau_t}{2}M_t Z_t (r + |\Psi_t(\theta_t)| ) \leq D_{\mathcal{R}} (v_t,\theta_t) - D_{\mathcal{R}} (v_t,\theta_{t+1}) +  \tau_t f_t^r(v_t).
    \end{align*}
    Since we know that
    \begin{align*}
        M_t\tau_t =  M_t\min\{C,\frac{f^r_t (\theta_t)}{\|\nabla \Psi_t(\theta_t)\|^2_*}\}\leq  M_t\min \{{C},\frac{r}{G^2}\}.
    \end{align*}
    Since $\tau_t\leq C.$ We therefore have
    \begin{align*}
       & \min \{C,\frac{r}{G^2}\}\frac{1}{2}M_t Z_t (r + |\Psi_t(\theta_t)| ) \leq D_{\mathcal{R}} (v_t,\theta_t) - D_{\mathcal{R}} (v_t,\theta_{t+1}) +  C f_t^r(v_t) .
    \end{align*}
    By the definition, we could infer that $p_t\leq |\Psi_t(\theta_t)|.$ Because if $\hat{y}=y$ then $p_t = |\Psi_t(\theta_t)|,$ and if $\hat{y}\neq y$ then $H_t^y(\theta_t) \leq H_t^{s_t}(\theta_t),$ where $s_t= \max_{s_t\neq\hat{y}_t, s_t\in \mathcal{Y}}H_t^{s_t}(\theta_t)$, which leads to 
    \begin{equation*}
        p_t = H_t^{\hat{y}} (\theta_t) - H_t^{s_t}(\theta_t) \leq H_t^{\hat{y}} (\theta_t) - H_t^y(\theta_t)= |\Psi_t(\theta_t)|.
    \end{equation*} 
    Thus we obtain
    \begin{align*}
       \min \{C,\frac{r}{G^2}\}\frac{1}{2}M_t Z_t (r + p_t ) & \leq \min \{C,\frac{r}{G^2}\}\frac{1}{2}M_t Z_t (r + |\Psi_t(\theta_t)| ) \\
       & \leq D_{\mathcal{R}} (v_t,\theta_t) - D_{\mathcal{R}} (v_t,\theta_{t+1}) +  C f_t^r(v_t) .
    \end{align*}
\end{proof}
Within the above lemmas, we are now ready to prove the Theorem~\ref{thm:labelErrorGeneralMulticlass}.
\begin{proof}[Proof of Lemma~\ref{thm:labelErrorGeneralMulticlass}]
    First by the proof of Lemma~1, we have
    \begin{align*}
        & \sum_{t=1}^T D_{\mathcal{R}} (v_t,\theta_t) - D_{\mathcal{R}} (v_t,\theta_{t+1})  \leq  \epsilon_v +  \gamma C_T.
    \end{align*}
    Then 
    \begin{align*}
        & \mathbb{E}[\sum_{t=1}^T M_t]  = \frac{1}{r} \mathbb{E}[\sum_{t=1}^T M_t Z_t (r + p_t ) ] \\
        & = \frac{2}{r^2} \max\{\frac{r}{C},{G^2}\} \mathbb{E}[\sum_{t=1}^T \min \{C,\frac{r}{G^2}\}\frac{1}{2}M_t Z_t (r + p_t )]\\ 
        & \leq \frac{2}{r^2} \max\{\frac{r}{C},{G^2}\} \mathbb{E}[\sum_{t=1}^T D_{\mathcal{R}} (v_t,\theta_t) - D_{\mathcal{R}} (v_t,\theta_{t+1})+ \sum_{t=1}^T C f_t^r(v_t)]\\ 
        & \leq \frac{2}{r^2} \max\{\frac{r}{C},{G^2}\}  (\epsilon_v + \gamma C_T + C \sum_{t=1}^T l_t^r(v_t))\\
        & = \frac{2G^2}{\sigma^2} (\gamma C_T + \epsilon_v + \frac{\sigma}{G^2} T \alpha^* ) ,
    \end{align*}
    where $r = \sigma,C = {\sigma}/{G^2}$. The second inequality comes from $l_t^r(v_t) = \mathbb{E}[f_t^r(v_t)]$, and the last inequality comes from $l_t^r(v_t) \leq l_t^R(v_t) \leq \alpha^*.$ We thus end the proof.
\end{proof}

\subsubsection{Regret Bound}
The regret bound analysis is actually the same as the binary case, since the Proposition~\ref{prop:TotalVariation} and Theorem~\ref{thm:generalBound} do not depend on the number of class. For contentedness, we present the result as follows.
\begin{theorem}[Regret Bound]\label{thm:ExpectedRegretNoiseGradientGeneralMulticlass}
  Under the same conditions and parameters in Lemma~\ref{thm:labelErrorGeneralMulticlass}. Algorithm~\ref{algorithm:MOAST} achieves the following regret bound
    \begin{align*}
      &\text{\rm D-Regret}^{\text{OSAMD}}(\{P_t\},T) \leq    \frac{4(\eta G^4 + G^3 D)}{\sigma^2}  (\gamma C_T + \epsilon_v + \frac{\sigma}{G^2} T \alpha^* ) + 2(LD+G)^2\eta T + \frac{\epsilon_w + \gamma D}{\eta} + 4\sqrt{\frac{\gamma D T F V_T }{\eta}}.
    \end{align*}
\end{theorem}
The proof is also the same as Theorem~4. From this, we know that our result still works in the multiclass case.

\section{Experimental Details}\label{supp:additionalexperiment}
In this section, we provide experimental details in Section~6.

\textbf{Datasets~~}
We provide detailed descriptions of our datasets as follows:
\begin{enumerate}[\hspace{0em}1.]
    \item \emph{Rotating Gaussian}: We simulate a non-stationary environment with continual domain shift. We use two Gaussian distributions with center points $(5,0)$ and $(15,0),$ and covariance matrix $3 I$ ($I$ denotes identity matrix), to represent class $1$ and $-1.$ We let the center points averagely rotate from $0\degree$ to $180\degree$ counterclockwise in $2000$ time steps, and in each time, we sample one data instance. Therefore, every data sample comes from a different domain. All the time, we keep $P(Y = 1) = P(Y = -1) = 1/2.$ 
    \item \emph{Rotating MNIST}: We randomly select and shuffle 35000 images from the original MNIST dataset, using the first 10,000 images with no rotation as the source dataset. We averagely rotate the next 25,000 images from $0\degree$ to $90\degree$ counterclockwise to be the target dataset with a continually changing domain.
    \item \emph{Portraits}: It is a realistic dataset, which contains 37,921 photos of high school seniors labeled by gender across many years. This real dataset suffers from a natural continual domain shift, including covariate shift and label shift, as shown in previous works~\citep{ginosar2015century,kumar2020understanding}. We downsample all the images to 32x32, and do no other preprocessing.  We take the first 2000 images as the source domain. We use the next 16000 images as target data with a continually changing domain, and test the online adaptation.
    \item \emph{Cover-Type}: It is a realistic dataset from the UCI repository. This goal is to predict the forest Cover-Type at a particular location given 54 features~\citep{blackard1999comparative}. The original Cover-Type dataset contains 581012 samples and has 7 type classes to be predicted. In our experiment, we leave the examples in the first two classes (which compose the majority of the dataset, have 500k samples in total) and sort the examples by increasing horizontal distance to the water body. Then we split the data into a source domain (first 50K examples), an intermediate domain (next 400K examples), and a target domain (final 50K examples). This dataset setting follows the setting in~\citet{kumar2020understanding}.
\end{enumerate}

\textbf{Baselines~~} We provide a detailed introduction about the baselines we compare. Since we are the first to study the OACA setting, no specific baseline is suitable for this setting. To demonstrate the efficacy of our design, we compare with the following baselines:
\begin{enumerate}[\hspace{0em}1.]
    \item \textit{Passive-aggressive active (PAA) learning:} The design of our aggressive model is similar to online active learning (OAL) algorithms. To demonstrate the advantage of the online teacher-student structure, we compare with one typical algorithm PAA. Although OAL is a well-studied topic of statistical learning, most recently proposed methods are not suitable for implementation in deep learning settings. For instance, the second-order OAL algorithm \citep{hao2017second} is designed only for the linear model, and requires additional computation cost of the second-order matrix. More recently, \citet{zhang2018online,zhang2019online} studied OAL with class imbalance, but provided no additional improvement in the balance cases. Therefore, comparing with PAA is sufficient to show the advantage against online active learning.
    \item \textit{Online mirror descent with all labels (OMD (all)):} Our theory shows that the regret of OSAMD is aligned with the lower bound for online learning with full labels. To verify the theoretical results, we compare with online mirror descent with all labels, which has been shown to attain the lower bound~\citep{besbes2015non,jadbabaie2015online}. By observing whether the regret or accumulated loss of OSAMD is aligned with OMD (all), we could empirically verify the theory.
    \item \textit{Online mirror descent with uniform sampled labels (OMD (partial)):} Online mirror descent with uniform sampled labels is the naive way to deal with the OACA problem. By comparing it with OMD (all), we could know whether the naive method can solve this problem. By comparing it with OSAMD, we can demonstrate the advantage of our sophisticated design.
\end{enumerate}
Next, we introduce the baselines for the ablation study. Recall the online teacher-student structure consists of self-adaptation and active query, we then compare with the following baselines to show the efficacy of each component:
\begin{enumerate}[\hspace{0em}1.]
    \item \textit{OSAMD without Self-adaptation:} To evaluate the efficacy of self-adaptation, we run a OMD with the same active queries as OSAMD. By comparing it with OSAMD, we could empirically observe the efficacy of the design of self-adaptation.
    \item \textit{OSAMD without Active-query:} To evaluate the efficacy of self-adaptation, we replace the active queries with uniform sampled labels in OSAMD. By comparing it with OSAMD, we could empirically observe the efficacy of the design of active-query.
\end{enumerate}

\textbf{Model and Parameters setting~~} We provide our model and parameters setting as follows:
\begin{enumerate}[\hspace{0em}1.]
    \item \emph{Models}: 
    \begin{itemize}
        \item For Rotating Gaussian, we set objective function to be the svm loss $f(w;x,y)=\max\{0,1-y w^T x\} + C\|w\|_2^2$ with penalty parameter $0.2,$ the soft prediction is $H(w) = w^T x.$
        \item For Rotating MNIST and Portraits, we design the same neural network feature extractor with two conv layers. We use filter size of 5×5, stride of 2×2, 64 output channels, and relu activation for each layer. After the final convolution layer, we add a dropout layer with a probability of 0.5 and a batchnorm layer after dropout. 
        The extracted features are then flattened and fed into fully connected layers with 2 and 10 outputs, respectively, for Portraits and Rotating MNIST. Each of the output neurons is matched with a specific prediction class.
        \item For Cover-Type, we used a two hidden layer feedforward. Each linear hidden layer contains 30 neurons. Dropout layer with a probability of 0.5 is added before the final fully connected layer. The final output is activated by softmax.
    \end{itemize} 
    \item \emph{Parameters}: 
    \begin{itemize}
        \item For Rotating Gaussian, the step size $\eta$ is set to be $0.01$ for both OMD and OSAMD. We set active controller $\sigma=0.35,$ and aggressive step size $\tau_t = \min\{1, \max\{0,1-y_t \theta_t^T x_t\}/\|x_t\|^2_2\}$ for OSAMD. We use $l_2$ norm as the regularizer $\mathcal{R}$, and initialize all the models with $[0.4,0,-4].$ For the implicit gradient update of self adaptation, we run 20 inner gradient descent loops to approximate the optimal.
        \item For Rotating MNIST, the step size $\eta$ is set to be $0.000005$ for both OMD and OSAMD (abbreviation for MOSAMD). We set active controller $\sigma=0.2,$ and aggressive step size $\tau_t = \min\{0.006, 0.0027*\max\{0,1-y_t \Psi_t(\theta_t)\}\}$ for OSAMD, where $\Psi_t(\theta_t)$ is the margin function of the deep learning model. $l_2$ norm is used as the regularizer $\mathcal{R}$ and all the models are initialized with a model pre-trained with the source data, i.e., the first 10000 images. For the implicit gradient update of self adaptation, we run 10 inner gradient descent loops to approximate the optimal.
        \item For Portraits, the step size $\eta$ is set to be $0.000001$ for both OMD and OSAMD. We set active controller $\sigma=0.15,$ and aggressive step size $\tau_t = \min\{0.0025, 0.0012*\max\{0,1-y_t H_t(\theta_t)\}\}$ for OSAMD, where $H_t(\theta_t)$ is the output of the deep learning model. $l_2$ norm is used as the regularizer $\mathcal{R}$ and all the models are initialized with a model pre-trained with the source data, i.e., the first 2000 images. For the implicit gradient update of self adaptation, we run 20 inner gradient descent loops to approximate the optimal.
        \item For Cover-Type, the step size $\eta$ is set to be $0.0000015$ for both OMD and OSAMD. We set active controller $\sigma=0.005,$ and aggressive step size $\tau_t = \min\{0.02, 0.01*\max\{0,1-y_t H_t(\theta_t)\}\}$ for OSAMD, where $H_t(\theta_t)$ is the output of the deep learning model. $l_2$ norm is used as the regularizer $\mathcal{R}$ and all the models are initialized with a model pre-trained with the source data, i.e., the first 50K examples. For the implicit gradient update of self adaptation, we run 5 inner gradient descent loops to approximate the optimal.
    \end{itemize}
\end{enumerate}
\textbf{Implementation:}
\begin{enumerate}[\hspace{0em}1.]
    \item \emph{Set Up:} The training and evaluation of models are realized with PyTorch (\url{https://pytorch.org}). We repeat every experiment over 10 times, and report the mean performance across independent runs. We also present the confidence intervals to eschew the experimental randomness.
    \item \emph{Computation Resources:} We have run the simulation on a single Intel(R) Xeon(R) E5-2650 CPU, and the deep learning experiments on a single 16GB GeForce GTX 1080 Ti GPU.
\end{enumerate}

\section{Social Impact}
For the social impact, as a study on a general learning problem, our work will not incur ethical issues by itself. However, ethical issues may arise if our learning method is improperly applied to some application fields - just as any other general learning method if it is misused.

\end{document}